\documentclass[letterpaper, conference]{ieeeconf}

\IEEEoverridecommandlockouts                              %

\usepackage{savesym}
\usepackage{amsmath,amssymb,amsfonts,amsthm}
\usepackage[ruled,  noend,rightnl]{algorithm2e}
\usepackage{comment}
\usepackage{xspace}
\usepackage{scalerel}
\usepackage{xcolor}
\usepackage{url} 
\usepackage[font=footnotesize]{caption}
\usepackage[font=footnotesize]{subcaption}
\usepackage{multicol}

\makeatletter
\let\NAT@parse\undefined
\makeatother
\usepackage[bookmarks=true]{hyperref}
\usepackage{bm}
\usepackage{nicefrac}
\usepackage{enumerate}
\usepackage{booktabs} %
\usepackage[numbers,sort&compress]{natbib} %

\usepackage{graphicx}
\usepackage[binary-units]{siunitx}
\usepackage{cleveref}
\hypersetup{bookmarksopen,bookmarksnumbered,
pdfpagemode=UseOutlines,
colorlinks=true,
linkcolor=red,
anchorcolor=red,
citecolor=red,
filecolor=red,
menucolor=red,
urlcolor=black
}

\def\BibTeX{{\rm B\kern-.05em{\sc i\kern-.025em b}\kern-.08em
				T\kern-.1667em\lower.7ex\hbox{E}\kern-.125emX}}

\renewcommand{\baselinestretch}{0.97} %

\title{\LARGE \bf Efficient Large-Scale Multi-Drone Delivery Using Transit Networks}

\author{Shushman Choudhury, Kiril Solovey, Mykel J. Kochenderfer, and Marco Pavone%
\thanks{The authors are with Stanford University, CA, USA.}
}

\begin{document}
	
\maketitle
\thispagestyle{empty}
\pagestyle{empty}

\def\frechet{Fr\'echet\xspace}

\newcommand{\cupdot}{\mathbin{\mathaccent\cdot\cup}}

\newcommand{\mtm}{\emph{multi-to-multi}\xspace}
\newcommand{\mts}{\emph{multi-to-single}\xspace}
\newcommand{\sts}{\emph{multi-to-single-restricted}\xspace}
\newcommand{\dtd}{\emph{single-to-single}\xspace}

\newcommand{\cte}{\emph{full-to-edge}\xspace}
\newcommand{\ctc}{\emph{full-to-full}\xspace}
\newcommand{\ete}{\emph{edge-to-edge}\xspace}

\newcommand{\AND}{{\sc and}\xspace}
\newcommand{\OR}{{\sc or}\xspace}

\newcommand{\ignore}[1]{}

\def\vor{\text{Vor}}

\def\P{\mathcal{P}} \def\C{\mathcal{C}} \def\H{\mathcal{H}}
\def\F{\mathcal{F}} \def\U{\mathcal{U}} \def\L{\mathcal{L}}
\def\O{\mathcal{O}} \def\I{\mathcal{I}} \def\S{\mathcal{S}}
\def\G{\mathcal{G}} \def\Q{\mathcal{Q}} \def\I{\mathcal{I}}
\def\T{\mathcal{T}} \def\L{\mathcal{L}} \def\N{\mathcal{N}}
\def\V{\mathcal{V}} \def\B{\mathcal{B}} \def\D{\mathcal{D}}
\def\W{\mathcal{W}} \def\R{\mathcal{R}} \def\M{\mathcal{M}}
\def\X{\mathcal{X}} \def\A{\mathcal{A}} \def\Y{\mathcal{Y}}
\def\L{\mathcal{L}} \def\E{\mathcal{E}}

\def\dS{\mathbb{S}} \def\dT{\mathbb{T}} \def\dC{\mathbb{C}}
\def\dG{\mathbb{G}} \def\dD{\mathbb{D}} \def\dV{\mathbb{V}}
\def\dH{\mathbb{H}} \def\dN{\mathbb{N}} \def\dE{\mathbb{E}}
\def\dR{\mathbb{R}} \def\dM{\mathbb{M}} \def\dm{\mathbb{m}}
\def\dB{\mathbb{B}} \def\dI{\mathbb{I}} \def\dM{\mathbb{M}}
\def\dZ{\mathbb{Z}} \def\dX{\mathbb{X}}

\def\eps{\varepsilon}

\def\limn{\lim_{n\rightarrow \infty}}

\def\obs{\mathrm{obs}}
\newcommand{\defeq}{%
  \mathrel{\vbox{\offinterlineskip\ialign{%
    \hfil##\hfil\cr
    $\scriptscriptstyle\triangle$\cr
    $=$\cr
}}}}
\def\Int{\mathrm{Int}}

\def\Reals{\mathbb{R}}
\def\Naturals{\mathbb{N}}
\renewcommand{\leq}{\leqslant}
\renewcommand{\geq}{\geqslant}
\newcommand{\compl}{\mathrm{Compl}}

\newcommand{\sig}{\text{sig}}

\newcommand{\sbs}{sampling-based\xspace}
\newcommand{\mr}{multi-robot\xspace}
\newcommand{\mpl}{motion planning\xspace}
\newcommand{\mrmp}{multi-robot motion planning\xspace}
\newcommand{\sr}{single-robot\xspace}
\newcommand{\cs}{configuration space\xspace}
\newcommand{\conf}{configuration\xspace}
\newcommand{\confs}{configurations\xspace}

\newcommand{\stl}{\textsc{Stl}\xspace}
\newcommand{\boost}{\textsc{Boost}\xspace}
\newcommand{\core}{\textsc{Core}\xspace}
\newcommand{\leda}{\textsc{Leda}\xspace}
\newcommand{\cgal}{\textsc{Cgal}\xspace}
\newcommand{\qt}{\textsc{Qt}\xspace}
\newcommand{\gmp}{\textsc{Gmp}\xspace}

\newcommand{\Cpp}{C\raise.08ex\hbox{\tt ++}\xspace}

\def\concept#1{\textsf{\it #1}}
\def\ccode#1{{\texttt{#1}}}

\newcommand{\ch}{\mathrm{ch}}
\newcommand{\pspace}{{\sc pspace}\xspace}
\newcommand{\threesum}{{\sc 3Sum}\xspace}
\newcommand{\np}{{\sc np}\xspace}
\newcommand{\argmin}{\operatornamewithlimits{argmin}}
\newcommand{\argmax}{\operatornamewithlimits{argmax}}

\newcommand{\Gdisk}{\G^\textup{disk}}
\newcommand{\Gbt}{\G^\textup{BT}}
\newcommand{\Gsoft}{\G^\textup{soft}}
\newcommand{\Gnear}{\G^\textup{near}}
\newcommand{\Gembed}{\G^\textup{embed}}

\newcommand{\dist}{\textup{dist}}

\newcommand{\Cfree}{\C_{\textup{free}}}
\newcommand{\Cforb}{\C_{\textup{forb}}}

\newtheorem{lemma}{Lemma}
\newtheorem{theorem}{Theorem}
\newtheorem{corollary}{Corollary}
\newtheorem{claim}{Claim}
\newtheorem{proposition}{Proposition}

\theoremstyle{definition}
\newtheorem{definition}{Definition}
\newtheorem{remark}{Remark}
\theoremstyle{plain}
\newtheorem{observation}{Observation}

\def\len{c_\ell}
\def\bot{c_b}

\def\lenopt{\len^*}
\def\botopt{\bot^*}

\def\Im{\textup{Im}}

\def\rfunc{\left(\frac{\log n}{n}\right)^{1/d}}
\def\rfuncs{\left(\frac{\log n}{n}\right)^{1/d}}
\def\cfunc{\sqrt{\frac{\log n}{\log\log n}}}
\def\rtrs{\gamma\rfunc}
\def\ctrs{2\cfunc}
\def\aconn{\A_\textup{conn}}
\def\abd{\A_\textup{str}}
\def\aspan{\A_\textup{span}}
\def\aopt{\A_\textup{opt}}
\def\ao{\A_\textup{ao}}
\def\acfo{\A_\textup{acfo}}
\def\binomial{\textup{Binomial}}
\def\twin{\textup{twin}}

\def\aas{a.a.s.\xspace}
\def\0{\bm{0}}

\def\distU#1{\|#1\|_{\G_n}^U}
\def\distW#1{\|#1\|_{\G_n}^W}

\def\tooth{\scalerel*{\includegraphics{./../fig/tooth}}{b}}

\makeatletter
\def\thmhead@plain#1#2#3{%
  \thmname{#1}\thmnumber{\@ifnotempty{#1}{ }\@upn{#2}}%
  \thmnote{ {\the\thm@notefont#3}}}
\let\thmhead\thmhead@plain
\makeatother

\def\todo#1{\textcolor{blue}{\textbf{TODO:} #1}}
\def\new#1{\textcolor{magenta}{#1}}
\def\kiril#1{\textcolor{magenta}{(\textbf{Kiril:} #1)}}
\def\old#1{\textcolor{red}{#1}}
\def\shushman#1{\textcolor{red}{\textbf{Shushman:} #1}}
\def\zak#1{\textcolor{orange}{\textbf{Zak:} #1}}

\def\removed#1{\textcolor{green}{#1}}

\def\dx{\,\mathrm{d}x}
\def\dy{\,\mathrm{d}y}
\def\drho{\,\mathrm{d}\rho}

\newcommand{\prm}{{\tt PRM}\xspace}
\newcommand{\prmstar}{{\tt PRM}$^*$\xspace}
\newcommand{\rrt}{{\tt RRT}\xspace}
\newcommand{\est}{{\tt EST}\xspace}
\newcommand{\grrt}{{\tt GEOM-RRT}\xspace}
\newcommand{\rrtstar}{{\tt RRT}$^*$\xspace}
\newcommand{\rrg}{{\tt RRG}\xspace}
\newcommand{\btt}{{\tt BTT}\xspace}
\newcommand{\fmt}{{\tt FMT}$^*$\xspace}
\newcommand{\dfmt}{{\tt DFMT}$^*$\xspace}
\newcommand{\dprm}{{\tt DPRM}$^*$\xspace}
\newcommand{\mstar}{{\tt M}$^*$\xspace}
\newcommand{\drrtstar}{{\tt dRRT}$^*$\xspace}
\newcommand{\sst}{{\tt SST}\xspace}
\newcommand{\aorrt}{{\tt AO-RRT}\xspace}

\newcommand{\mct}{{MCT}\xspace}

\let\oldnl\nl%
\newcommand{\nonl}{\renewcommand{\nl}{\let\nl\oldnl}}%

\newcommand{\alg}{\textsc{MergeSplitTours}\xspace}
\newcommand{\opt}{\textup{\textsc{opt}}\xspace}

\newcommand{\timeCost}{T\xspace}
\newcommand{\energy}{N\xspace}
\newcommand{\capacity}{C\xspace}
\newcommand{\depot}{d\xspace}
\newcommand{\package}{p\xspace}
\newcommand{\depotSet}{D\xspace}
\newcommand{\packageSet}{P\xspace}
\newcommand{\depotVerts}{V_{\depotSet}\xspace}
\newcommand{\packageVerts}{V_{\packageSet}\xspace}
\newcommand{\transitVerts}{V_{TN}\xspace}
\newcommand{\transitRoute}{R\xspace}
\newcommand{\speed}{\sigma\xspace}

\def\x{\bm{x}}

\newif\ifarxiv
\arxivtrue
\ifarxiv
\newcommand{\supp}[1]{Appendix #1}
\else
\newcommand{\supp}[1]{the appendix}
\fi

\begin{abstract}
We consider the problem of controlling a large fleet of drones to deliver packages 
simultaneously across broad urban areas. To conserve energy, drones hop between public transit vehicles (e.g., buses and trams). 
We design a comprehensive algorithmic framework that strives to minimize the maximum
time to complete any delivery. We address the multifaceted complexity of the problem 
through a two-layer approach. First, the upper layer assigns drones to
package delivery sequences with a near-optimal polynomial-time task allocation
algorithm. Then, the lower layer executes the allocation by periodically routing the fleet over the transit network while employing efficient bounded-suboptimal multi-agent pathfinding techniques tailored to our setting.
Experiments demonstrate the efficiency of our approach on settings with up to $\bm{200}$ drones, $\bm{5000}$ packages,
and transit networks with up to $\bm{8000}$ stops in San Francisco and Washington DC.
Our results show that the framework computes solutions typically within a few seconds on commodity hardware, and that
drones travel up to $\bm{360 \%}$ of their flight range with public transit.

\end{abstract}

\section{Introduction}
Rapidly growing e-commerce demands have greatly strained dense urban communities
by increasing delivery truck traffic and slowing operations and impacting
travel times for public and private vehicles~\cite{Humes18,HolguinETAL18}.
Further congestion is being induced by newer services
relying on ride-sharing vehicles. There is a clear need to redesign the current method of package distribution
in cities~\cite{FakleETAL17}. 
The agility and aerial reach of drones, the flexibility and ease of
establishing drone networks, and recent advances in drone
capabilities make them highly promising for logistics
networks~\cite{JoerssETAL16}. However, drones have limited travel range and carrying capacity~\cite{SudburyETAL16,Dji}.
On the other hand, ground-based transit networks have less flexibility but greater coverage and throughput.
By combining the strengths of both, we can achieve significant commercial benefits
and social impact (e.g., reducing ground congestion and delivering essentials).

\begin{figure}[t]
  \centering
  \fbox{\includegraphics[width=0.95\columnwidth]{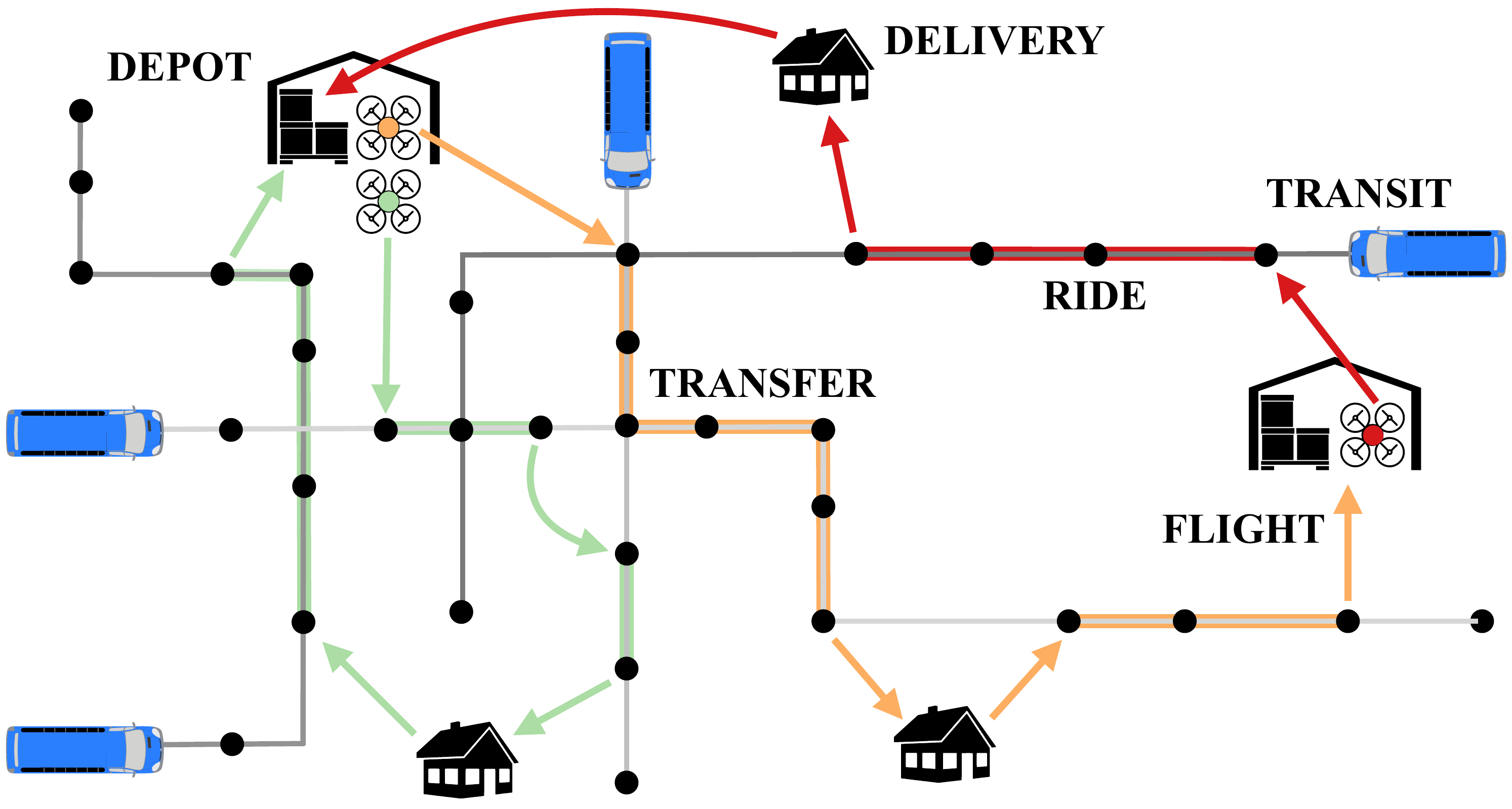}}
  \caption{
  Our multi-drone delivery framework plans for drones to piggyback on public transit vehicles
  while delivering packages from depots to the requested locations.
  Our framework is scalable and efficient, and minimizes the amount of time for any individual delivery.}
  \label{fig:pg1-fig}
  \vspace{-15pt}
\end{figure}

We address the problem of operating a large number of drones to deliver multiple packages 
simultaneously in an area. The drones can use one or more 
vehicles in a public-transit network
as modes of transportation, thereby saving  their 
limited battery energy stored onboard and increasing their effective travel range. We are required to 
decide which deliveries each drone should make and in what 
order, which modes of transit to use, and for what duration (\Cref{fig:pg1-fig}). 

Our approach must contend with the multiple significant challenges of our problem.
It must plan over large time-dependent transit networks, while
accounting for energy constraints that limit the drones' flight ranges. It must avoid inter-drone
conflicts, such as where more than one drone
attempts to board the same vehicle at the same time, or when the maximum carrying capacity of a vehicle
is exceeded. We seek not just feasible multi-agent plans
but high-quality solutions in terms of a cumulative objective over all drones, the makespan, i.e.,
the maximum individual delivery time for any drone. Additionally, our approach must also solve the task allocation problem of determining
which drones deliver which packages, and from which distribution centers.

\subsection{Related work}
Some individual aspects of our problem have already been studied.~\citet{choudhury2019dynamic} investigated the single-agent setting of controlling a drone to use 
multiple modes of transit en route to its destination. 
Recent work has considered pairing a drone with a delivery truck, which does not exploit public transit~~\cite{MurrayChu15,AgatzETAL18,FerrandezETAL16}. 
The multi-agent issues of task allocation and inter-agent conflicts were not
addressed either.
Our problem is closely related to routing a fleet of autonomous vehicles providing mobility-on-demand services~\cite{SoloveyETAL19,IglesiasETAL19,WallerETAL18}. Specifically, the task is to compute routes for the vehicles (both customer-carrying and empty) so that travel demand is fulfilled and operational cost is minimized. In particular, recent works study the combination of such service with public transit, where passengers can use several modes of transportation in the same trip~\cite{SalazarETAL18, ZgraggenETAL19}.
However, such works abstract away inter-agent constraints or dynamics and are not suited for autonomous pathfinding. 
The task-allocation setting we consider in our problem can be viewed as an instance of the vehicle routing problem~\cite{CaceresETAL14,AlenaETAL18, TothVigo2014}, variants of which are typically solved by mixed integer linear programming (MILP) formulations that scale poorly, or by heuristics without optimality guarantees. 

We must contend with the challenges of planning for multiple agents.
Accordingly, the second layer of our approach is a multi-agent path finding (MAPF) problem~\cite{erdmann1987multiple, YuLaValle16}.
Since the drones are on the same team, we have a centralized or cooperative pathfinding setting~\cite{silver2005cooperative}.
The MAPF problem is NP-hard to solve optimally~\cite{yu2013structure}. A number of efficient solvers have been developed that
work well in practice~\cite{felner2017search}. The MAPF formulation and algorithms have been extended to several relevant scenarios such as lifelong pickup-and-delivery~\cite{ma2017lifelong} and joint task assignment and pathfinding~\cite{honig2018conflict,liu2019task}, though for different task
settings and constraints than ours. 
Also, a MAPF formulation was applied for UAV traffic management in cities~\cite{ho2019multi}.
However, none of the approaches considered pathfinding over large time-dependent transit
networks. We use models, algorithms and techniques from transportation planning~\cite{pyrga2008efficient,delling2009engineering,bast2016route}.

\subsection{Statement of contributions}
We present a comprehensive algorithmic framework for large-scale multi-drone delivery in synergy with a ground transit network. 
Our approach strives to minimize the maximum time to complete any delivery.
We decompose the highly challenging problem and solve it stage-wise
with a two-layer approach. First, the upper layer assigns drones to
package-delivery sequences with a %
task allocation algorithm. Then, the lower layer executes the allocation by periodically routing the fleet over the transit network. %

Algorithmically, we develop a new delivery sequence allocation method for the upper layer that obtains a near-optimal solution in polynomial runtime. 
For the lower layer, we extend techniques for multi-agent path finding that account for
time-dependent transit networks and agent energy constraints to perform multi-drone routing. 
Experimentally, we present results supporting the efficiency of our approach on settings with up to ${200}$ drones, ${5000}$ packages, and transit networks of up to ${8000}$ stops in San Francisco and the Washington DC area.
Our framework can compute solutions within a few seconds (up to ${2}$ minutes for the largest settings) on commodity hardware, and in our problem scenarios,
drones can travel up to ${450 \%}$ of their flight range using transit.

The following is the paper structure. We present an overall description of the two-layer approach in \Cref{sec:methodology}, and then elaborate on each layer in Sections~\ref{sec:allocation} and~\ref{sec:mapf-tn}. We present experimental results on simulations in \Cref{sec:results}, and conclude the paper with Section~\ref{sec:conclusion}.
\section{Methodology}
\label{sec:methodology}
We provide a high-level description of our formulation and approach to illustrate the various interacting components.

\subsection{Problem Formulation}
\label{sec:methodology-problem}

We are operating a centralized homogeneous fleet of $m$ drones within a city-scale domain.
There are $\ell$ product \emph{depots} with known geographic locations, denoted by
$V_D:=\{d_1,\ldots,d_\ell\} \subset \dR^2$. The depots are both product dispatch centers
and drone-charging stations.
At the start of a large time interval (e.g., a day), a batch of delivery request locations for $k$ different \emph{packages},
denoted $V_P:=\{p_1,\ldots,p_k\}\subset \dR^2$,
is received (we assume that $k \gg m$).
We assume that any package can be dispatched from any depot; our framework exploits this property to optimize
the solution quality in terms of \emph{makespan}, i.e., the maximum execution time for any delivery. In~\Cref{sec:allocation}, we mention
how our approach can accommodate dispatch constraints.

The drones carry packages from depots to delivery locations.
They can extend their effective travel range by using public transit vehicles in the area, which remain unaffected by the drones' actions.
Our problem is to route drones to deliver all packages while minimizing makespan.
A drone route consists of its current location and the sequence of depot and package locations to visit
with a combination of flying and riding on transit.
We characterize the drones' limited energy as a maximum flight distance constraint.
A feasible solution must satisfy inter-drone constraints such as collision avoidance and
transit vehicle capacity limits.

Finally, we make some assumptions for our setting: a drone carries one package at a time, which is reasonable given state-of-the-art drone payloads~\cite{Dji}; drones are recharged upon visiting a depot in negligible time (e.g., a battery replacement); depots have unlimited drone capacity; the transit network is deterministic with respect to locations and vehicle travel times 
(we mention uncertainty in~\Cref{sec:conclusion}). We do account for the time-varying nature of the transit.

\subsection{Approach overview}
\label{sec:methodology-approach}

In principle, we could frame the entire problem as a mixed integer linear program (MILP). 
However, for real-world problems (hundreds of drones; thousands of packages; large transit networks),
even state-of-the-art MILP approaches are unlikely to scale.
Moreover, even a simpler problem that ignores the interaction constraints
is an instance of the notoriously challenging multi-depot vehicle routing problem~\cite{AlenaETAL18}.
Thus, we decouple the problem into two distinct subproblems that we solve stage-wise in layers.

The upper layer performs \emph{task allocation} to determine which packages are delivered by which drone
and in what order. It takes as input the known depot and package locations, and an estimate
of the drone travel time between every pair of locations.
It then solves a threefold allocation to minimize delivery makespan and assigns to each package 
(i) the dispatch depot and (ii) the delivery drone, and to each drone (iii) the order of package deliveries.
To this end, we develop an efficient polynomial-time task-allocation algorithm that achieves a near-optimal makespan.

The lower layer performs \emph{route planning} for the drone fleet to execute the allocated delivery tasks.
It generates detailed routes of drone locations in time and space and the transit vehicles used, while 
accounting for the time-varying transit network. It also ensures that
(i) simultaneous transit boarding by multiple drones is avoided, (ii) no transit vehicle exceeds
its drone-carrying capacity, and (iii) drone (battery) energy constraints are respected.
We efficiently handle individual and inter-drone constraints 
by framing the routing problem as an extension of multi-agent path finding (MAPF) 
to transit networks. We adapt a scalable, bounded sub-optimal variant of a highly effective MAPF
solver called Conflict-Based Search (CBS)~\cite{sharon2012conflict} to solve the one-delivery-per-drone problem.
Finally, we obtain routes for the sequence of deliveries in a receding-horizon fashion by replanning for the next task
once a drone completes its current one.

Decomposition-based stage-wise optimization approaches typically have an approximation gap compared to the optimal solution of the full problem.
For us, this gap manifests in the surrogate cost estimate we use
for the drone's travel time in the task-allocation layer
(instead of jointly solving for allocation and multi-agent routing over transit networks, which 
is not feasible at scale). The better the surrogate, the more \emph{coupled} the layers are, i.e.,
the better is the solution of the first stage for the second one. Such surrogates have a tradeoff between efficiency and
approximation quality.
An easy-to-compute travel time surrogate, for instance, is the drone's direct flight
time between two locations (ignoring transit). However, that can be poor-quality
when the drone requires transit for an out-of-range target.
We use a surrogate that actually accounts for the transit network,
at the expense of some modest preprocessing. 
We defer details to \supp{\ref{sec:appendix-surrogate}}, but the idea is to precompute
the pairwise shortest travel times between locations spread around 
the city, over a representative snapshot of the transit network.

\section{Task Assignment and Package Allocation}
\label{sec:allocation}

We leverage our problem's structure to design a new algorithm called \alg for the
task-allocation layer, which guarantees a near-optimal solution in polynomial time.
The goal of this layer is to (i) distribute the set of packages $V_P$ among $m$ agents, 
(ii) assign each package destination $p\in V_P$ to a depot $d\in V_D$, and
(iii) assign drones to a sequence of depot pickups and package deliveries.
The objective is to \textbf{minimize the maximum travel time among all agents} over all three of the above components.

Our problem can be cast as a special version of the $m$ traveling salesman problem~\cite{Bektas06}, which we call the $m$ \emph{minimal
visiting paths problem} ($m$-MVP). We seek a set of $m$ paths such that the makespan, i.e., the maximum travel time for any path, is minimized. 
We only need \emph{paths} that start and end at (the same or different) depots, not tours. Our formulation is a special case of the the \emph{asymmetric} variant, for a \emph{directed} underlying graph, which is NP-hard even for $m = 1$ on general graphs~\cite{AsadpourETAL17} (although it is not known whether the specific instance of our problem is NP-hard as well). Moreover, the current best polynomial-time approximation~\cite{AsadpourETAL17}
yields the fairly large approximation factor $O(\log n / \log \log n)$, for a graph with $n$ vertices.
An additional challenge is the inability to assume the triangle inequality on our objective of travel times.%

A key element of $m$-MVP is the \emph{allocation graph} $G_A=(V_A,E_A)$, with vertex set $V_A=V_D\cup V_P$.
Each directed edge $(u,v) \in E_A$ is weighted according to an estimated travel time $c_{uv}$ from the location of $u$
to that of $v$ in the city.  For every $d\in V_D, p\in V_P$ we exclude the edge $(d,p)$ from $E_A$ if it is impossible to reach $p$ from $d$ while using at most $1/2$ of the flight range allowed (similarly for $(p,d)$ edges).
As we flagged in~\Cref{sec:methodology-approach}, any dispatch constraints
are modeled by excluding edges from the corresponding depot.
We are now ready for the full definition of $m$-MVP:

\begin{algorithm}[t]
  \SetNlSkip{0em}
  Solve $\textsc{mct}(G_A)$ to get $t$ tours $\dT:=\{T_1,\ldots,T_t\}$\;
  \While{$|\dT|>1$}{
    Pick distinct tours $T,T'\in \dT$ and depots $d\in T,d'\in T'$ that minimize 
    $c_{dd'}+c_{d'd}$\;
    Merge $T,T'$ by adding $(d,d'),(d',d)$ edges \;
  }
  Split final tour $T$ into $m$ paths $P_1,\ldots, P_m $, where $\textsc{length}(P_i)$ is proportional to $\textsc{length}(T)/m$ for each $i$ (similar to~\cite{FredericksonETAL76})\;
  Extend each $P_i$ to ensure it begins and ends at a depot\;
  \Return $P_1,\ldots, P_m$\;
  \caption{\textsc{MergeSplitTours}$(G_A)$}
  \label{alg:mst_short}
  \vspace{-2pt}
\end{algorithm} 

\begin{table}
\caption{An integer programming formulation of the MCT problem.}
\noindent\fbox{\begin{minipage}[t]{0.95\columnwidth}%
Given allocation graph $G_A=(V_A,E_A)$, with $V_A=V_D\cup V_P$, 
\begin{align}
{\textrm{minimize}}   \sum_{(u,v)\in E_A}x_{uv}\cdot c_{uv}  \label{eq:mct:obj}
\end{align}\vspace{-5pt}
subject to
\begin{align}
   x_{uv}\in \{0,1\},  \quad  \forall (u,v)\in E_A, u\in V_P\vee v\in V_P, \\
  x_{uv}\in \dN_{\geq 0},  \quad \forall (d,d')\in E_A, d,d'\in V_D, \\
  \sum_{d\in N_+(p)}x_{dp}= \sum_{d\in N_-(p)}x_{pd}=1, \quad \forall p\in V_P, \label{eq:mct:out} \\
  \sum_{v\in N_+(d)}x_{vd}-\sum_{v\in N_-(d)}x_{dv}=0,  \quad \forall d\in V_D. \label{eq:mct:equal}
\end{align}
where $N_+(v), N_-(v)$ denote the in and out going neighbors of $v\in V_A$.
\end{minipage}}\vspace{5pt}
\label{tbl:mct}
\vspace{-15pt}
\end{table}

\begin{definition}\label{def:mvp}
Given allocation graph $G_A$, the $m$ minimal visiting paths problem ($m$-MVP) consists of finding $m$ paths $P^{*}_{1:m}$ on $G_A$, such that (1) each path $P^*_i$ starts at some depot $d\in V_D$ and terminates at the same or different $d'\in V_D$, (2) exactly one path visits each package $p\in V_P$, and (3) the maximum travel time of any of the paths is minimized. 
\end{definition}
Let $\textsc{opt}$ be the optimal makespan, i.e., $\textsc{opt}:=\max_{i\in [m]}\textsc{length}(P^*_i)$, where $\textsc{length}(\cdot)$ denotes the total travel time along a given path or tour.
We make three observations. First, if a path contains the sub-path $(d,p),(p,d')$, for some $d,d'\in V_D,p\in V_P$, 
then $p$ should be dispatched from depot $d$ and the drone delivering $p$ will return to $d'$ after delivery.
Second, a package $p$ being found in $P^*_i$ indicates that drone $i\in [m]$ should deliver it. Third,~$P^*_i$ fully characterizes the order of packages delivered by drone~$i$.

\begin{figure*}[th]
  \centering
  \begin{subfigure}{0.245\textwidth}
    \centering
    \fbox{\includegraphics[width=0.97\textwidth]{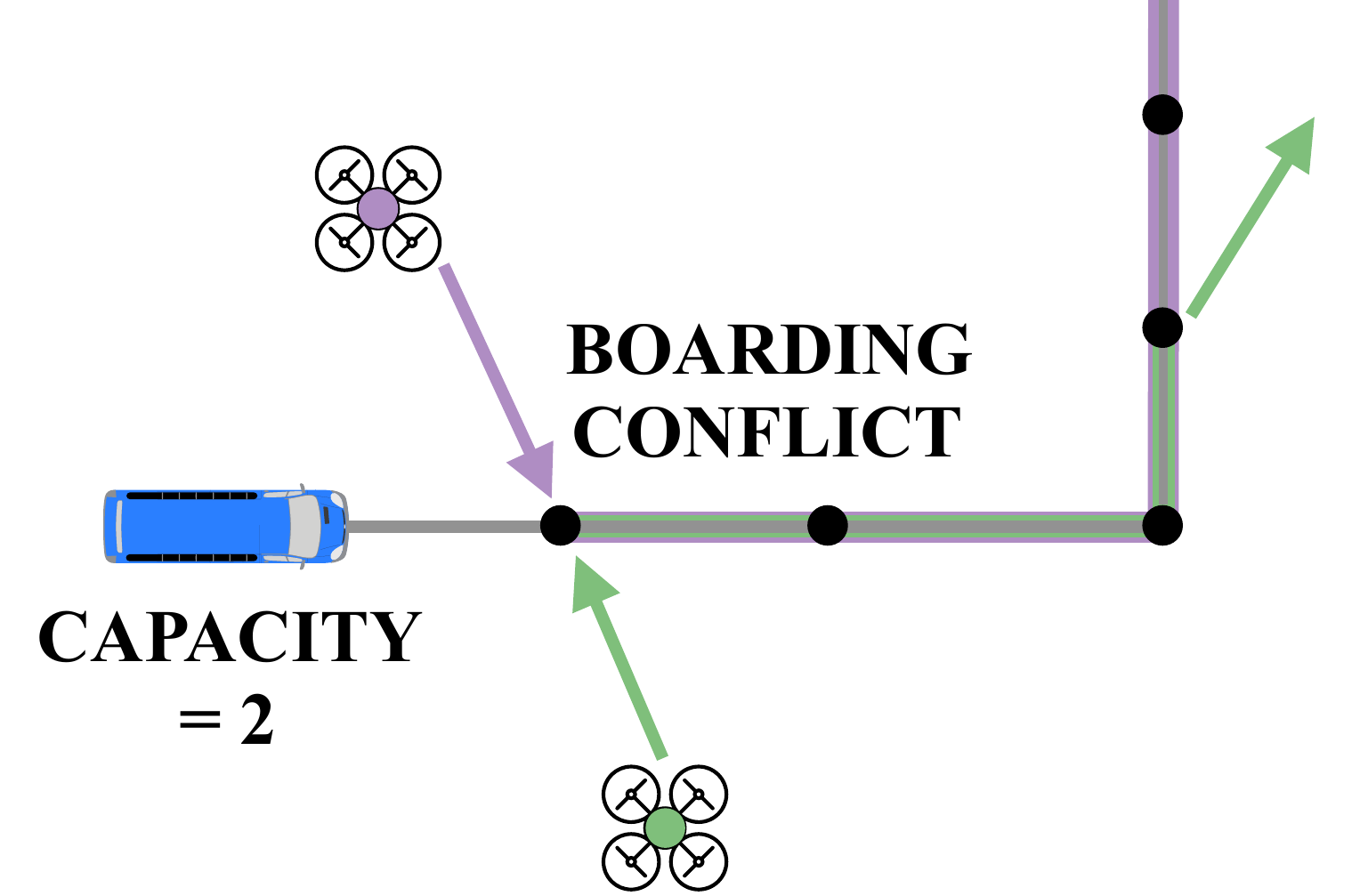}}
    \caption{}
    \label{fig:pre-conflict-1}
  \end{subfigure}
  \begin{subfigure}{0.245\textwidth}
    \centering
    \fbox{\includegraphics[width=0.97\textwidth]{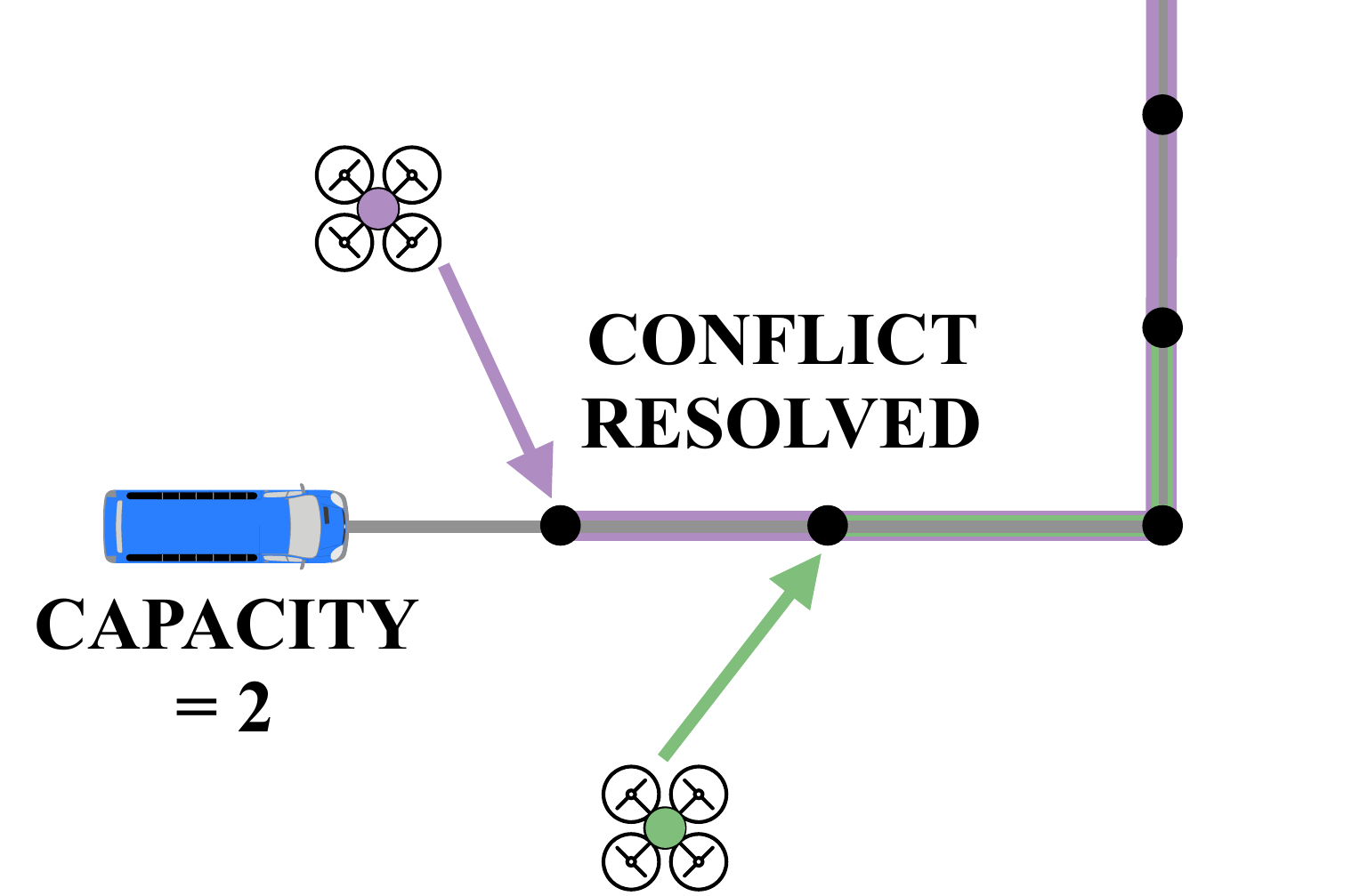}}
    \caption{}
    \label{fig:post-conflict-1}
  \end{subfigure}
  \begin{subfigure}{0.245\textwidth}
    \centering
    \fbox{\includegraphics[width=0.97\textwidth]{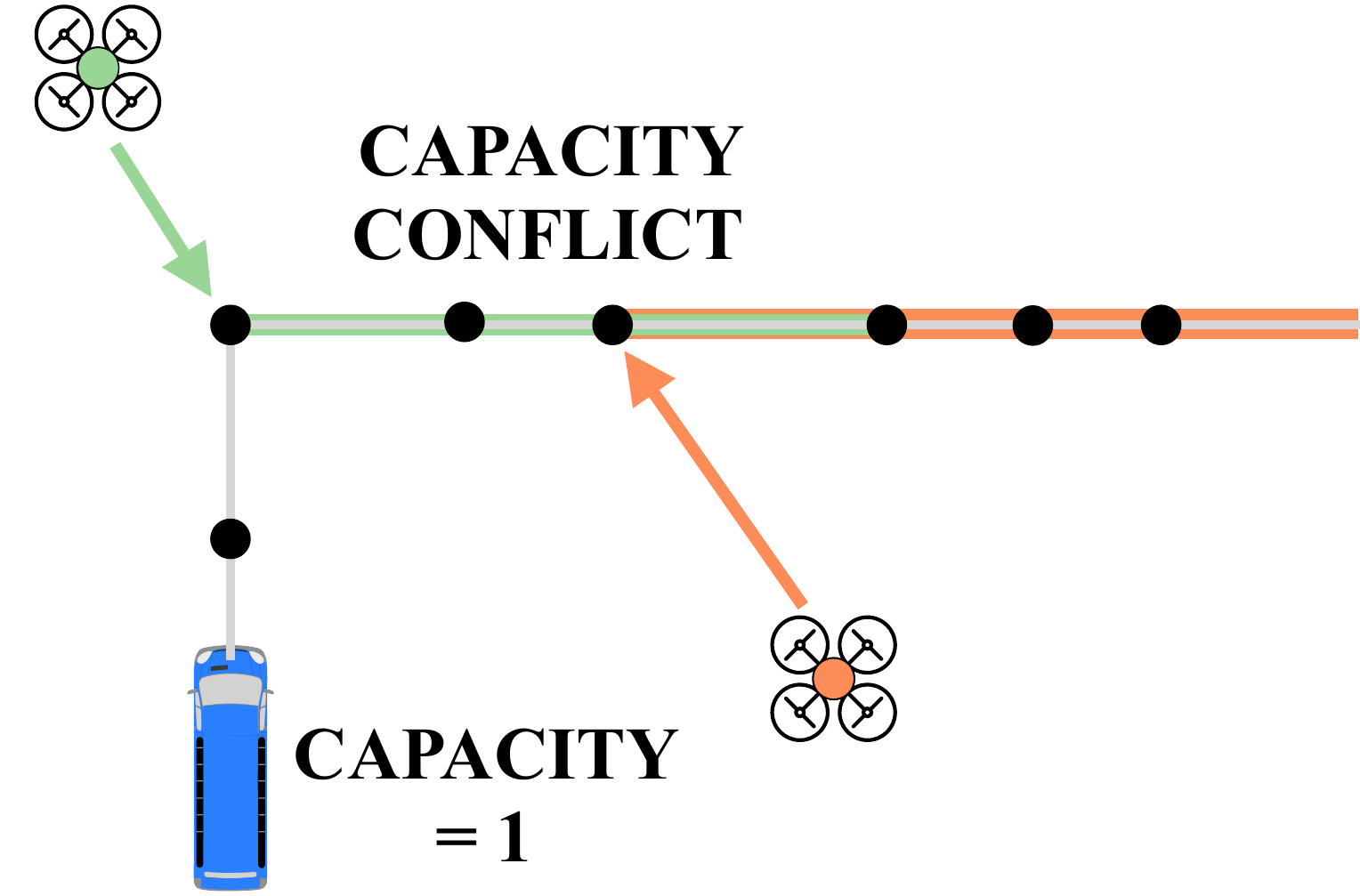}}
    \caption{}
    \label{fig:pre-conflict-2}
  \end{subfigure}
  \begin{subfigure}{0.245\textwidth}
    \centering
    \fbox{\includegraphics[width=0.97\textwidth]{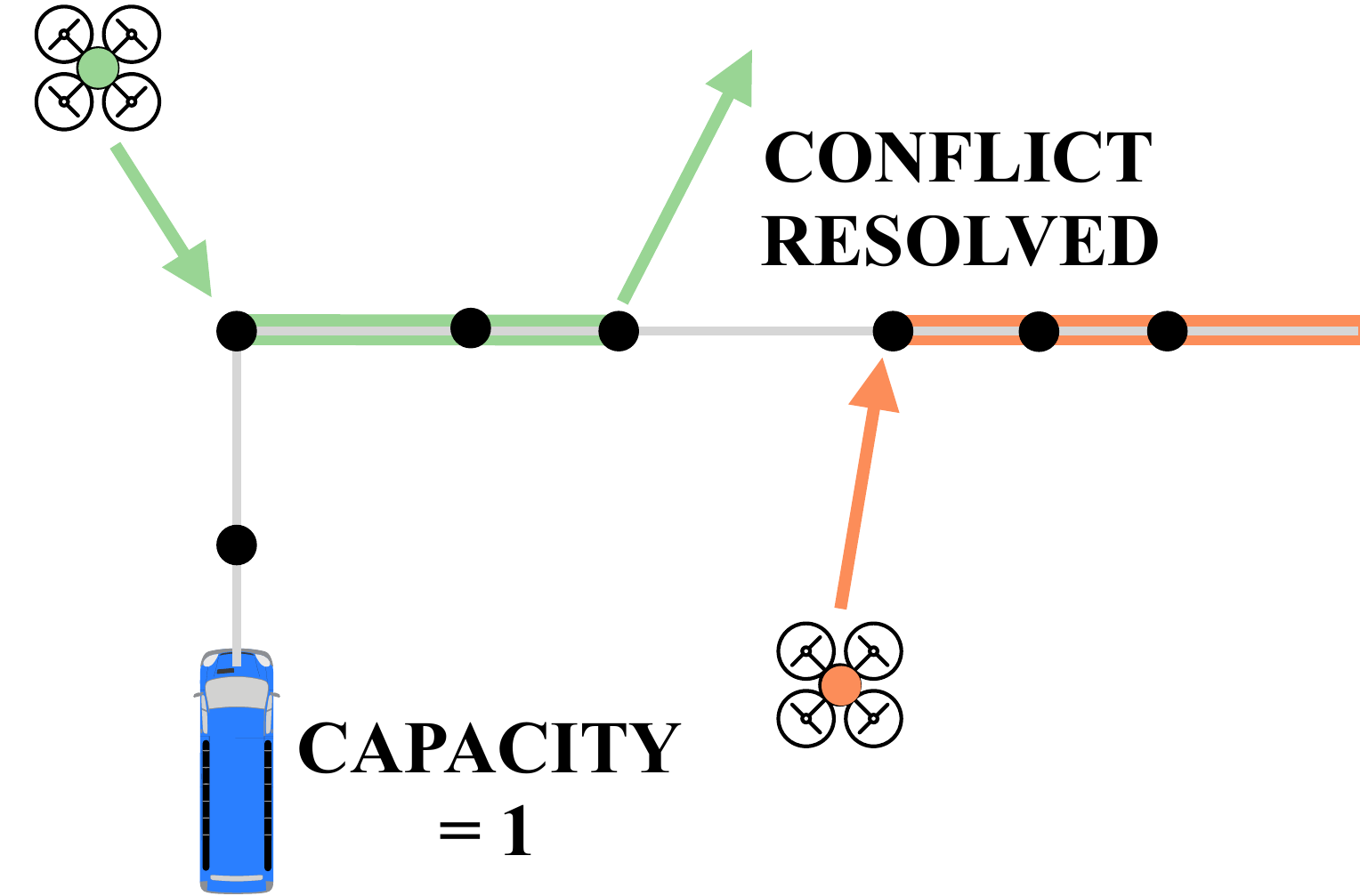}}
    \caption{}
    \label{fig:post-conflict-2}
  \end{subfigure}
  \caption{In our formulation of multi-agent path finding with transit networks, conflicts arise from the violation of
  shared inter-drone constraints: (\subref{fig:pre-conflict-1}) boarding conflicts between two
  or more drones and (\subref{fig:pre-conflict-2}) capacity conflicts between more drones than the
  transit vehicle can accommodate. The modified paths after resolving the corresponding conflicts are depicted in (\subref{fig:post-conflict-1})
  and (\subref{fig:post-conflict-2}), respectively.}
  \label{fig:conflicts}
  \vspace{-10pt}
\end{figure*}

\subsection{Algorithm Overview}
We present our \alg algorithm for solving $m$-MVP (\Cref{alg:mst_short}); see a detailed description  in \supp{\ref{sec:appendix-allocation}}. 
A key step is generating an initial set of tours $\dT$ by solving the minimal-connecting tours (\mct) problem (see~\Cref{tbl:mct}),
which attempts to connect packages to depots within tours to minimize the total edge weight in~\cref{eq:mct:obj}.
The constraint in~\cref{eq:mct:out} is that each package is connected to precisely one incoming and one outgoing edge from and to depots respectively.
The final constraint in~\cref{eq:mct:equal} enforces inflow and outflow equality for every depot.
Edges connecting packages can be used at most once,
whereas edges connecting depots can be used multiple times. The solution to \mct is the assignment $\{x_{uv}\}_{(u,v)\in E_A}$,
i.e., which edges of $G_A$ are used and how many times. This assignment implicitly represents the desired collection of the tours $T_1,\ldots, T_t$; see \supp{\ref{sec:appendix-allocation}}.

\subsection{Theoretical Guarantees}
All proofs from this secion are in \supp{\ref{sec:appendix-allocation}}.
The following theorem states that \alg is correct and that its makespan is close to optimal.

\begin{theorem}\label{thm:complete}
Suppose $G_A$ is strongly connected and the subgraph $G_A(V_D)$ induced by the vertices $V_D$ is a directed clique. Let $P_1,\ldots,P_m$ be the output of \alg. 
Then, every package $p\in V_P$ is contained in exactly one path $P_i$, and every $P_i$  starts and ends at a depot. Moreover, $\max_{i\in [m]}\textsc{length}(P_i) \leq \textsc{opt} + \alpha+\beta$ holds,
\[\textup{where } \alpha:=\max_{d,d'\in V_D} c_{dd'}+c_{d'd}\ , \ \beta:=\max_{d,d'\in V_D, p\in V_P}c_{dp}+c_{pd'}.\]
\end{theorem} 

The key idea is that the total cost of the tours induced by the solution to \mct cannot exceed the total length of $\{P_1^*,\ldots,P_m^*\}$. The \mct solution is then adapted to $m$ paths with an additional overhead of $\alpha+\beta$ per path. When $m\ll|V_P|$ (typically the case), $\alpha \text{ and } \beta$ are small compared to $\opt$, making the bound tight.
For instance, in our randomly-generated scenarios in~\Cref{sec:results-alloc}, for $m = 5 \text{ and } k = 200$, the 
approximation ratio $\max_{i \in [m]} \textsc{length}(P_i) / \textsc{opt} = 1.09$, and for $m = 10, k = 500$,
the factor is $1.06$. %

The computational bottleneck of the algorithm is \mct,
while the other components can clearly 
be implemented polynomially in the input size. However, it suffices to solve a relaxed version of \mct to obtain the same integral solution. 

\begin{lemma}\label{lem:poly}
  The optimal solution to the fractional relaxation of \mct, in which $x_{uv}\in [0,1]$ for all $u\in V_P\vee v\in V_P$, and $x_{uv}\in \dR_+$ otherwise, yields the integer optimal solution. %
\end{lemma}

The lemma follows from casting \mct as the minimum-cost circulation problem, for which the constraint matrix is totally unimodular~\cite{AhujaETAL93}. Therefore, \alg can be implemented in polynomial time.

\section{Multi-Agent Path Finding}
\label{sec:mapf-tn}
For each drone $i \in [m]$, the allocation layer yields a sequence of delivery tasks
$\depot_{1}\package_{1}\ldots\package_{l}\depot_{l+1}$.
Each delivery sequence has one or more subsequences of
$\depot \package\depot^{\prime}$. 
The route-planning layer treats each $\depot \package\depot^{\prime}$ subsequence 
\emph{as an individual
drone task}, i.e., leaving with the package from depot $\depot$,
carrying it to package location $\package$ and returning to the
(same or different) depot $\depot^{\prime}$, without exceeding the energy
capacity. We seek an efficient and scalable method to obtain high-quality (with respect to travel time) feasible paths, while  using transit options to extend range,
for $m$ different drone $\depot\package\depot^{\prime}$ tasks simultaneously.
The full set of delivery sequences can be satisfied
by replanning when a drone finishes its current task and begins a new one;
we discuss and compare two replanning strategies in \supp{\ref{sec:appendix-results}}.
Thus, we formulate the problem of multi-drone routing to satisfy a set of delivery sequences as receding-horizon 
multi-agent path finding (MAPF) over transit networks. In this section, we describe the graph representation
of our problem and present an efficient bounded sub-optimal algorithm.

\subsection{MAPF with Transit Networks (MAPF-TN)}
\label{sec:mapf-tn-formulation}

The problem of 
Multi-Agent Path Finding with Transit Networks (MAPF-TN) is the extension of
standard MAPF to where agents can use one or more modes of transit
in addition to moving.
The incorporation of transit networks introduces additional challenges and underlying structure.
The \textbf{input} to MAPF-TN is the set of $m$ tasks $(d_i,p_i,d'_i)_{i=1:m}$ and the
directed operation graph $G_O = (V_O,E_O)$. 
In~\Cref{sec:allocation}, the allocation graph $G_A$ only considered depots and packages, and edges between them. 
Here, $G_O$ also includes transit vertices, $\transitVerts = \bigcup_{\tau \in \T} \transitRoute_{\tau}$, where $\T$ is the set
of trips, and each trip $\transitRoute_{\tau} = \{(s_1, t_1) \ldots\}$ is a sequence of time-stamped stop locations 
(a given stop location may appear as several different nodes with distinct time-stamps). Similarly, we also use time-expanded
versions of $V_D$ and $V_P$~\cite{pyrga2008efficient}.

The edges are defined as follows: An edge $e = (u , v) \in E$ is a \emph{transit edge} if $u,v \in \transitVerts$ and are consecutive
stops on the same trip $\transitRoute_t$. Any other edge is a \emph{flight edge}.
An edge is
\emph{time-constrained} if $v \in V_{TN}$ and \emph{time-unconstrained} otherwise.
Every edge has three attributes:
traversal time $\timeCost$, energy expended $\energy$, and capacity $\capacity$. Since each vertex is associated
with a location, $\lVert v - u \rVert$ denotes the distance between them for a suitable metric.
MAPF typically abstracts away agent dynamics; we have a simple model where drones move at constant speed $\speed$, and 
distance flown represents energy expended. 
Due to high graph density (drones can fly point-to-point
between many stops), we do not explicitly enumerate edges but generate them on-the-fly during search.

We now define the three attributes for $E_O$.
For time-constrained edges, $\timeCost(e) = v.t - u.t$ is the difference between corresponding time-stamps 
(if $u \in \depotVerts \cup \packageVerts$, $u.t$ is the chosen departure time), and for time-unconstrained edges, $\timeCost(e) = \lVert v - u \rVert  / \speed$
is the time of direct flight.
For flight edges, $\energy(e) = \lVert v - u \rVert$ (flight distance), and for transit edges,
$\energy(e) = 0$. For transit edges, $\capacity(e)$ is bounded by the capacity of the vehicle, 
while for flight edges, $\capacity(e) = \infty$. Here, we assume that time-unconstrained flight in open space can be accommodated (thorougly examined in~\cite{ho2019multi}).

We now describe the remaining relevant MAPF-TN details.
An individual path $\pi_i$ for drone $i$ from $\depot_i$ through $\package_i$
to $\depot^{\prime}_i$
is \textbf{feasible} if the energy constraint $\sum_{e \in \pi_i} \energy(e) \leq \bar{\energy}$ is satisfied,
where $\bar{\energy}$ is the drone's maximum flight distance.
In addition, 
the drone should be able to traverse the distance of a time-constrained flight edge
in time,
i.e., $\speed \times (v.t - u.t) > \lVert v - u \rVert$. 
For simplicity, we abstract away energy expenditure due to hovering in place
by flying the drone at reduced speed to reach the transit just in time.
Thus, the constraint $\bar{\energy{}}$ is only on the traversed distance.
The cost of an individual path is the total
traversal time, $\timeCost(\pi_i) = \sum_{e \in \pi_i} \timeCost(e)$. 
A \textbf{feasible solution} $\Pi = \bigcup_{i=1:m} \pi_i$ is a set of $m$ individually feasible paths that does not violate any of the following two \emph{shared constraints} (see \Cref{fig:conflicts}): (i) \emph{Boarding constraint}, i.e., no two drones may board the same vehicle
at the same stop; (ii) \emph{Capacity constraint}, i.e., a transit edge $e$ may not be used by more than $\capacity(e)$ drones.
As with the allocation layer, the \textbf{global objective} for MAPF-TN is to minimize
the solution makespan, $\argmin_{\Pi} \max_{\pi \in \Pi} \timeCost(\pi)$,
i.e., minimize the worst individual completion time.

\subsection{Conflict-Based Search for MAPF-TN}
\label{sec:mapf-tn-cbs}

To tackle MAPF-TN, we modify the Conflict-Based Search (CBS) algorithm~\cite{sharon2012conflict}.
The multi-agent level of CBS identifies shared constraints and imposes corresponding path constraints on the single-agent level.
The single-agent level computes optimal individual paths that respect all constraints.
If individual paths conflict (i.e., violate a shared constraint),
the multi-agent level adds further constraints to resolve the conflict,
and invokes the single-agent level again, for the conflicting agents. In MAPF-TN, conflicts arise from boarding and capacity constraints.
CBS obtains optimal multi-agent solutions without having to run 
(potentially significantly expensive) multi-agent searches.
However, its performance can degrade heavily with many conflicts in which constraints are violated.~\Cref{fig:conflicts} 
illustrates the generation and resolution of conflicts in our MAPF-TN problem.

For scalability, we use a bounded sub-optimal variant of
CBS called \emph{Enhanced CBS} (ECBS), which achieves orders of magnitude speedups over CBS~\cite{barer2014suboptimal}.
ECBS uses bounded sub-optimal \emph{Focal Search}~\cite{pearl1982studies}
at both levels, instead of best-first A*~\cite{hart1968formal}. Focal search
allows using an inadmissible heuristic that prioritizes efficiency.
We now describe a crucial modification to ECBS required for MAPF-TN. 

\textbf{Focal Weight-constrained Search:}
Unlike typical MAPF, the low-level graph search in MAPF-TN has a path-wide constraint (traversal distance)
\emph{in addition to the objective function} of traversal time. 
For the shortest path problems on graphs, adding a path-wide constraint
makes it NP-hard~\cite{garey1990computers}.
Several algorithms for constrained search
require an explicit enumeration of the edges~\cite{dumitrescu2003improved,carlyle2008lagrangian}.
We extend the A* for MultiConstraint Shortest Path (A*-MCSP) algorithm~\cite{li2007fast} (suitable for our implicit graph)
to focal search (called Focal-MCSP). Focal-MCSP uses admissible heuristics on both objective and constraint
and maintains only non-dominated paths to intermediate nodes. This extensive book-keeping
requires a careful implementation for efficiency.

Focal-MCSP inherits the properties of A*-MCSP and Focal Search; therefore, it yields a bounded-suboptimal feasible path 
to the target. Accordingly, \textbf{ECBS with Focal-MCSP
yields a bounded sub-optimal solution to MAPF-TN}. The result follows
from the analysis of ECBS~\cite{barer2014suboptimal}.
Also, note that a $\depot \package \depot^{\prime}$ path requires a bounded sub-optimal
path from $\depot$ to $\package$ and another from $\package$ to $\depot^{\prime}$, such that their concatenation
is feasible. Since this is even more complicated, in practice, we run Focal-MCSP twice (from $\depot$ to $\package$ and $\package$ to $\depot^{\prime}$)
with half the energy constraint each time and concatenate the paths, guaranteeing feasibility.
In \supp{\ref{sec:appendix-speedup}} we discuss other required modifications to standard MAPF and important speedup techniques
that nonetheless retain the bounded sub-optimality of Enhanced CBS for our MAPF-TN formulation.

\begin{table}[t]
\caption{
The mean computation time for \alg in seconds, over $100$ different trials for each setting.
\alg is polynomial in input size and highly scalable. Here, $k = | \packageVerts |$ is 
the number of package deliveries and $\ell = |\depotVerts|$ is the number of depots.
For all instances that took longer than \SI{60}{\second}, only one trial was used.}
\centering
\begin{tabular}{@{} lrrrrr @{}}
    \toprule
    $k$ & $\ell=2$ & $\ell = 5$ & $\ell = 10$ & $\ell = 20$ & $\ell = 30$\\
    \midrule
    $50$   & $0.004$ & $0.016$   & $0.057$ & $0.248$  & $0.658$  \\
    $100$  & $0.012$ & $0.050$   & $0.195$ & $0.807$  & $2.117$  \\
    $200$  & $0.038$ & $0.173$   & $0.699$ & $2.968$  & $8.409$  \\
    $500$  & $0.201$ & $1.025$   & $4.384$ & $25.04$  & $85.01$  \\
    $1000$ & $0.781$ & $4.109$   & $24.30$ & $122.3$  & $322.5$ \\
    $5000$ & $23.97$ & $238.9$   & $1031$  & $2192$   & $5275$ \\
    \bottomrule
\end{tabular}
\label{table-alloc}
\vspace{-10pt}
\end{table}

\begin{table*}[th]
\centering
\caption{(All times are  in seconds) An extensive analysis of the MAPF-TN layer, on $100$ trials for each setting of depots and agents (and $30$ trials 
  for  $5$ depots and $50$ agents). Each trial uses different randomly generated depots and delivery locations.  The integer carrying capacity of any transit edge $\capacity(e)$
was randomly chosen from $\{3,4,5\}$ (single and double-buses). The sub-optimality factor for ECBS was $1.1$. For settings with $\nicefrac{m}{\ell} = 10$, a number of
trials timed out (over \SI{180}{s}) and were discarded.}
\begin{tabular}{@{} lrrrrcrrrr  @{}}
  \toprule
  & \multicolumn{4}{c}{\textbf{San Francisco} $\left(\lvert \transitVerts \rvert = 4192\  ; \text{Area } \SI{150}{\kilo \metre \squared}\right)$} & \phantom{a} & \multicolumn{4}{c}{\textbf{Washington DC} $\left(\lvert \transitVerts \rvert = 7608; \text{Area } \SI{400}{\kilo \metre \squared}\right)$}\\
  \cmidrule{2-5} \cmidrule{7-10} $\{$Depots, Agents$\}$&    $\{$Median, Avg$\}$ & $\{\text{Avg, Max}\}$ & $\{\text{Avg, Max}\}$ & Avg Soln. & &  $\{$Median, Avg$\}$ & $\{\text{Avg, Max}\}$   & $\{\text{Avg, Max}\}$  & Avg Soln.\\
                                   $\{\ell,m\}$          &  Plan  Time & Range Ext. & Transit Used & Makespan & & Plan Time  & Range Ext. & Transit Used & Makespan\\
  \midrule
  $\{5, 10\}$    &  $\{0.61, 1.17\}$    & $\{1.53, 3.41\}$   & $\{2.93, 6\}$   & $2554.7$  & &  $\{3.91, 5.65\}$ & $\{1.66, 3.08\}$  & $\{3.18, 7\}$  & $5167.3$\\
  $\{5, 20\}$    &  $\{1.39, 2.13\}$    & $\{1.61, 2.66\}$   & $\{3.48, 6\}$   & $2886.8$  & &  $\{9.01, 13.1\}$ & $\{1.79, 3.21\}$  & $\{3.57, 8\}$  & $5384.5$\\
  $\{5, 50\}$    &  $\{2.13, 3.89\}$    & $\{1.64, 2.48\}$   & $\{4.2, 6\}$   & $3380.9$  & &  $\{19.1, 28.9\}$ & $\{2.07, 3.21\}$  & $\{\textcolor{red}{\mathbf{4.44}}, 7\}$  & $6140.2$\\
  $\{10, 20\}$   &  $\{0.41, 1.02\}$    & $\{1.24, 2.35\}$   & $\{2.31, 6\}$   & $2091.6$  & &  $\{1.61, 4.67\}$ & $\{1.37, 3.12\}$  & $\{2.57, 7\}$  & $4017.2$\\
  $\{10, 50\}$   &  $\{0.73, 1.46\}$    & $\{1.38, \textcolor{red}{\mathbf{3.58}}\}$   & $\{2.94, 5\}$   & $2504.7$  & &  $\{4.77, 15.8\}$ & $\{1.72, 3.03\}$  & $\{3.53, 7\}$  & $5312.3$\\
  $\{10, 100\}$  &  $\{2.09, 7.29\}$    & $\{1.43, 2.16\}$   & $\{3.67, 8\}$   & $2971.8$  & &  $\{18.1, 26.2\}$ & $\{1.86, 3.18\}$  & $\{4.25, 8\}$  & $5623.9$\\  
  $\{20, 50\}$   &  $\{0.17, 0.46\}$    & $\{0.98, 1.69\}$   & $\{1.09, 7\}$   & $1273.6$  & &  $\{0.73, 1.92\}$ & $\{1.29, 2.88\}$  & $\{2.23, 7\}$  & $3571.8$\\
  $\{20, 100\}$  &  $\{0.49, 1.05\}$    & $\{1.06, 1.79\}$   & $\{1.61, \textcolor{red}{\mathbf{9}}\}$   & $1642.4$  & &  $\{2.45, 5.24\}$ & $\{1.48, 2.67\}$  & $\{3.19, 6\}$  & $4304.5$\\
  $\{20, 200\}$  &  $\{0.89, 2.10\}$    & $\{1.13, 2.31\}$   & $\{2.23, 6\}$   & $1898.5$  & &  $\{4.68, 10.5\}$ & $\{1.61, 2.87\}$  & $\{3.58, 7\}$  & $5085.6$\\
  \bottomrule
\end{tabular}
\label{table-mapf-res}
\vspace{-10pt}
\end{table*}

\section{Experiments and Results}
\label{sec:results}

We implemented our approach using the Julia language and tested it on a machine with a $6$-core \SI{3.7}{\giga\hertz} \SI{16}{\gibi\byte} RAM CPU.\footnote{The code for our work is available at \url{https://github.com/sisl/MultiAgentAllocationTransit.jl}.}
For very large combinatorial optimization problems, solution quality and algorithm efficiency
are of interest.
We have already shown that the upper and lower layers are near-optimal and bounded-suboptimal respectively in
terms of solution quality, i.e., makespan. 
Therefore, for evaluation we focus on their efficiency
and scalability to large real-world settings.
We do not attempt to baseline against a MILP approach for the full problem;
we estimate that a typical setting of interest will have on the order of $10^7$ variables 
in a MILP formulation, besides exponential constraints.

We ran simulations with two large-scale public transit networks in San Francisco (SFMTA)
and the Washington Metropolitan Area (WMATA).
We used the open-source General Transit Feed Specification data~\cite{gtfs} for each network.
We considered only the bus network 
(by far the most extensive), 
but our formulation can accommodate multiple modes.
We defined a geographical bounding box in each case, of area \SI{150}{\kilo \metre \squared} for SFMTA
and \SI{400}{\kilo \metre \squared} for WMATA (illustrated in \supp{\ref{sec:appendix-results}}),
within which depots and package locations were randomly generated.
For the transit network, we considered all bus trips that operate within the bounding box.
The \emph{size} of the time-expanded network, $\lvert \transitVerts \rvert$, is the total number of stops made by all trips;
$\lvert \transitVerts \rvert = 4192$ for SFMTA and $\lvert \transitVerts \rvert = 7608$ for WMATA (recall that edges are
implicit, so $\lvert E_{TN} \rvert$ varies between queries, but the full graph $G_O$ can be dense).
The drone's flight range constraint is set (conservatively) to $\SI{7}{\kilo\metre}$ and average speed to $\SI{25}{kph}$,
based on the DJI Mavic 2 specifications~\cite{Dji}.
In this section, we evaluate the two main components --- the task allocation
and multi-agent path finding layers. In \supp{\ref{sec:appendix-results}} we compare
the performance of two replanning strategies for when a drone finishes its current delivery, and
two surrogate travel time estimates for coupling the layers.

\subsection{Task Allocation}
\label{sec:results-alloc}

The scale of the allocation problem is determined by the
number of depots and packages, 
i.e., $\ell + k$.
The runtimes for \alg with varying $\ell, k$ over SFMTA are displayed in~\Cref{table-alloc}. %
The roughly quadratic increase in runtimes along a specific row or column
demonstrate that our provably near-optimal \alg algorithm is indeed
polynomial in the size of the input.
Even for up to $5000$ deliveries,
the absolute runtimes are quite reasonable.
We do not compare with naive MILP even for allocation,
as the number of variables would exceed $(\ell \cdot k)^2$, in addition to
the expensive subtour elimination constraints~\cite{miller1960integer}.

\subsection{MAPF with Transit Networks (MAPF-TN)}
\label{sec:results-ecbs}

Solving multi-agent path finding optimally is NP-hard~\cite{yu2013structure}.
Previous research has
benchmarked CBS variants and shown that Enhanced CBS
is most effective~\cite{barer2014suboptimal,cohen2016improved}. Therefore, we focus on extensively evaluating our own approach rather than
redundant baselining.~\Cref{table-mapf-res} quantifies several aspects 
of the MAPF-TN layer with varying numbers of depots $(\ell)$ and agents $(m)$, the two most tunable parameters. 
Before each trial, we run the allocation layer and collect $m$ $\depot \package \depot^{\prime}$ tasks, one for each agent. We then run the MAPF-TN solver
on this set of tasks to compute a solution.

We discuss broad observations here and provide a detailed analysis in
\supp{\ref{sec:appendix-results}}. The results are very promising; our approach scales to large numbers of agents ($200$)
and large transit networks (nearly $8000$ vertices); the highest average
makespan for the true delivery time is less than an hour (\SI{3380.9}{\second}) for SFMTA and
2 hours (\SI{6140.2}{\second}) for WMATA; drones are using up to $9$ transit options per route
to extend their range by up to $3.6$x. As we anticipated, \textbf{conflict resolution is a major bottleneck of MAPF-TN}.
A \emph{higher ratio of agents to depots} increases conflicts due
to shared transit, thereby increasing plan time (compare $\{5,20\}$ to $\{10,20\}$).
\emph{A higher number of depots} puts more
deliveries within flight range of a depot, reducing conflicts, makespan, and the need for
transit usage and range extension (compare $\{10,50\}$ to $\{20,50\}$).
Plan times are much higher for WMATA due to a larger area and a larger and less uniformly
distributed bus network, leading to higher single-agent search times and more multi-agent conflicts.
Trials taking more than $3$ minutes were discarded;
two pathological cases with SFMATA and WMATA (each with $\{l = 10, m = 100\}$) took nearly $4$ and $8$ minutes,
due to $30$ and $10$ conflicts respectively. In any case, a deployed system would have
better compute and parallelized implementations. 
Finally, note that the running times reported here are actually pessimistic, because we consider
cases where drones are released simultaneously from the depots, which increases conflicts.
However, a gradual release by executing the MAPF solver over a longer horizon (as we discuss in \supp{\ref{sec:appendix-results-replanning}}) results
in fewer conflicts, allowing us to cope with an even larger drone fleet.

\section{Conclusion and Future Work}
\label{sec:conclusion}

We designed a comprehensive algorithmic framework for solving the highly challenging problem
of multi-drone package delivery with routing over transit networks.
Our two-layer approach is efficient and highly scalable to large
problem settings and obtains high-quality solutions that satisfy the many
system constraints. We ran extensive simulations with two real-world transit networks that 
demonstrated the widespread applicability of our framework and how using ground transit judiciously
allows drones to significantly extend their effective range.

A key future direction is to perform case studies that estimate the operational cost of our framework, evaluate its impact
on road congestion, and consider potential externalities like noise pollution and 
disparate impact on urban communities. Another direction is to extend our model to overcome its limitations: delays
and uncertainty in the travel pattern of transit vehicles~\cite{muller2007timetable} and delivery time windows~\cite{solomon1987algorithms}; jointly routing
ground vehicles and drones; optimizing for the placements of depots, whose locations are currently
randomly generated and given as input.

\section*{Acknowledgments}
This work was supported in part by NSF, Award Number: 1830554, the Toyota Research Institute (TRI), and the Ford Motor Company. The authors thank Sarah Laaminach, Nicolas Lanzetti, Mauro Salazar, and Gioele Zardini for fruitful discussions on transit networks.

\bibliographystyle{plainnat}
\bibliography{bibliography}

\appendices

\section{Task Allocation: Additional Details and Proofs}
\label{sec:appendix-allocation}
We present a full and extended version of the \alg algorithm (\Cref{alg:mst}) for the task allocation layer.~\Cref{fig:alloc}
illustrates the behaviour of \mct, which provides an approximate solution for the $m$-MVP problem (\Cref{def:mvp}).
The algorithm consists of three main steps:  \vspace{5pt}

\begin{figure*}[t]
  \centering
  \begin{subfigure}{0.325\textwidth}
    \centering
    \fbox{\includegraphics[width=0.97\textwidth]{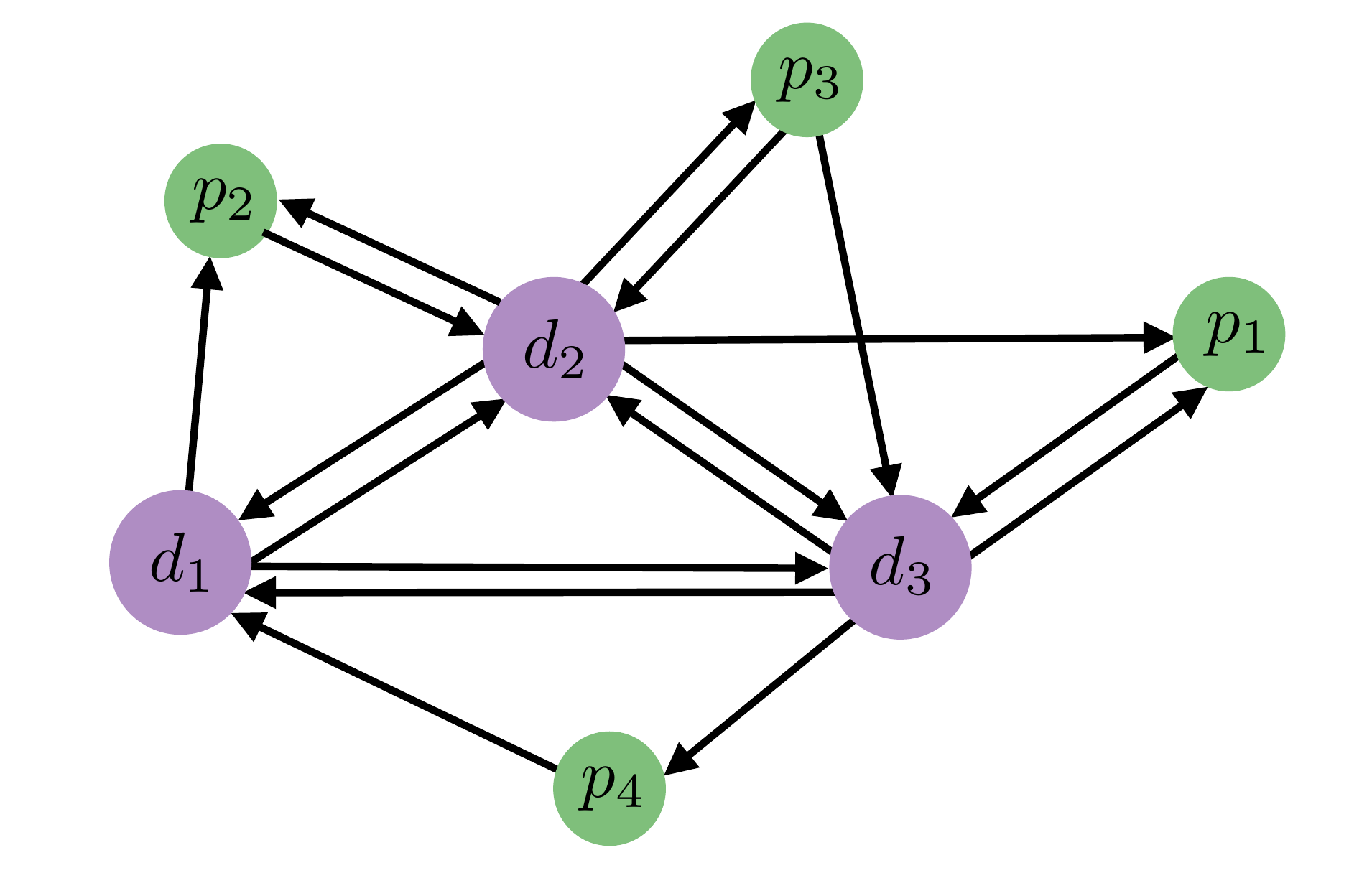}}
    \caption{}
    \label{fig:alloc-graph}
  \end{subfigure}
  \begin{subfigure}{0.325\textwidth}
    \centering
    \fbox{\includegraphics[width=0.97\textwidth]{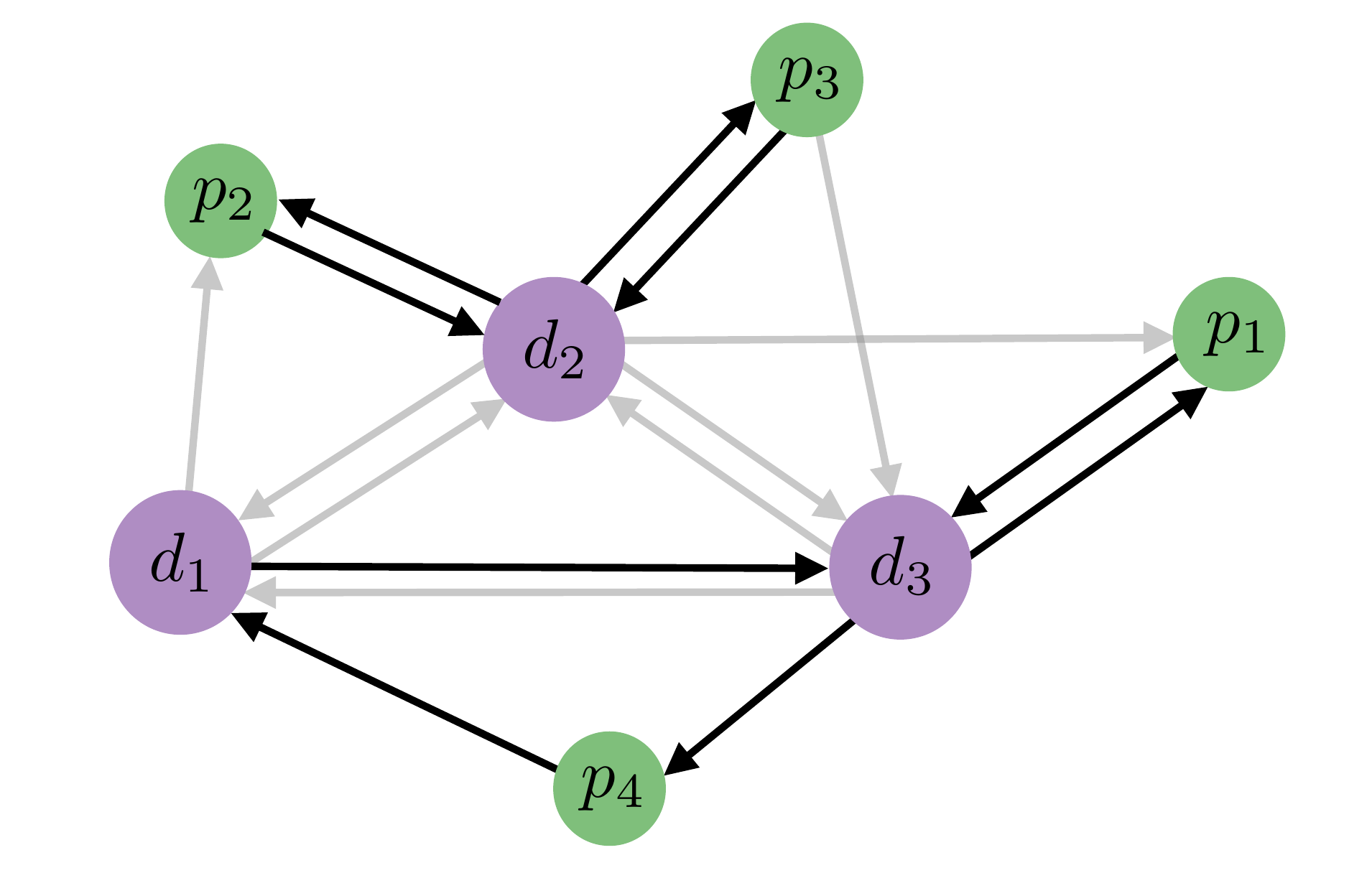}}
    \caption{}
    \label{fig:alloc-mct}
  \end{subfigure}
  \begin{subfigure}{0.325\textwidth}
    \centering
    \fbox{\includegraphics[width=0.97\textwidth]{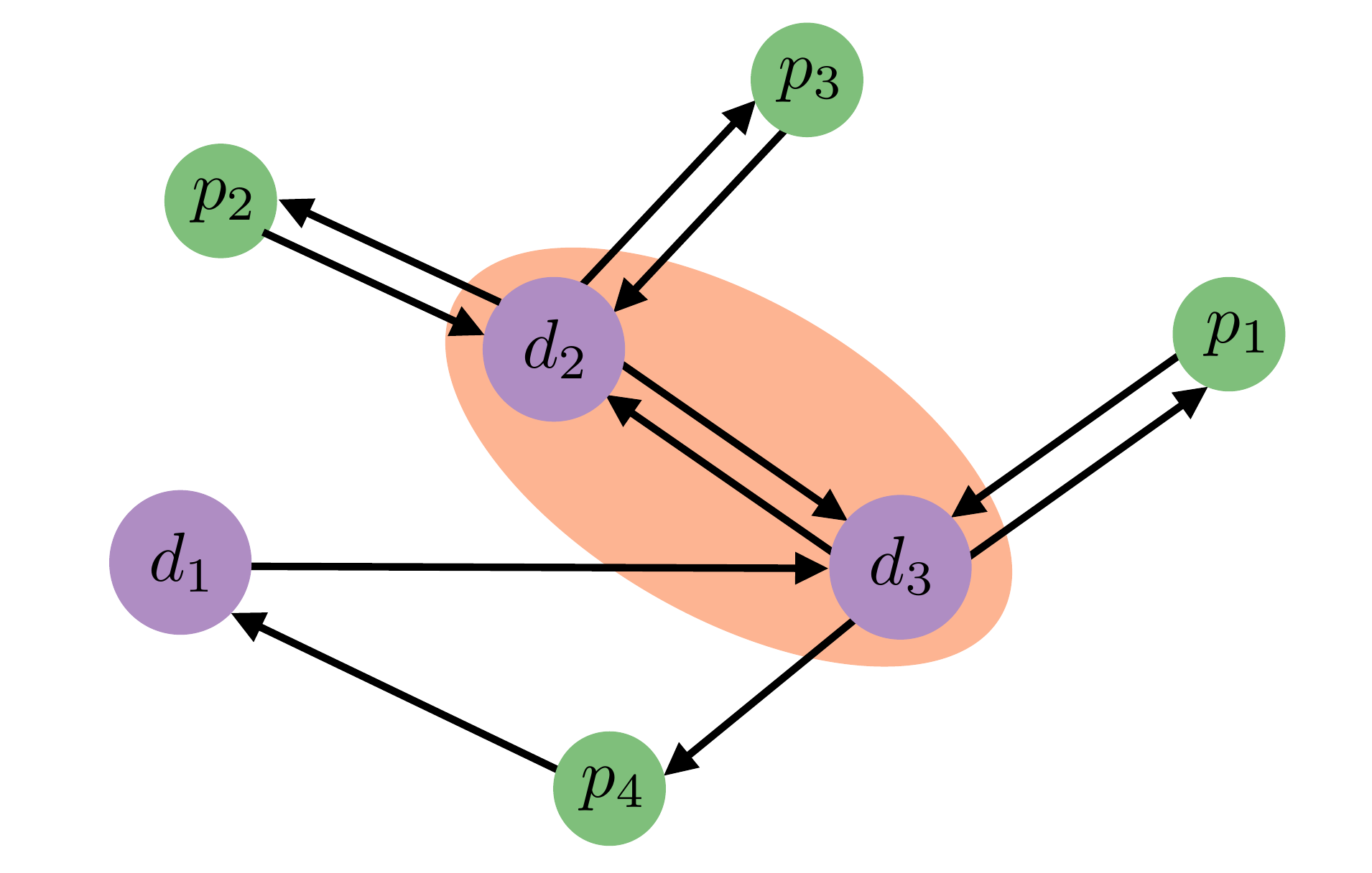}}
    \caption{}
    \label{fig:alloc-cluster}
  \end{subfigure}
  \caption{Three key steps in the \alg algorithm for delivery sequence allocation: (\subref{fig:alloc-graph}) The allocation graph
  is defined for the depots $\depot_{1:3}$ and packages $\package_{1:4}$. (\subref{fig:alloc-mct}) The \mct step yields a solution
  that connects each package delivery with the depot from which the drone is dispatched and the depot to which it returns.
  (\subref{fig:alloc-cluster}) A tour merging steps merges the depots $\depot_2$ and $\depot_3$ into a single cluster.}
  \label{fig:alloc}
\end{figure*}

\noindent\textbf{Step 1 (lines~1): } Generate a collection of $t$ tours $T_1,\ldots,T_t$, for some $1\leq t\leq k$, such that every package $p\in V_P$ is covered by exactly one tour, and the total distance of the tours is minimized. This step is achieved by solving the minimal-connecting tours (\mct) problem (see~\Cref{tbl:mct}). The solution to \mct is given by an assignment $\{x_{uv}\}_{(u,v)\in \dE}$, which indicates which edges of $G$ are used and 
for how many times. This assignment implicitly represents the desired collection of tours $T_1,\ldots, T_t$, as described above. The reason behind why such an assignment breaks into a collection of tours is discussed in \Cref{lem:tours} below. \vspace{5pt}

\noindent\textbf{Step 2 (lines~2-10): } The  $T_1,\ldots,T_t$ tours are merged in an iterative fashion, until a single tour $T$ is generated. We first identify $t\geq 1$ connected depot sets $\dD=\{D_1,\ldots, D_t\}$, which are induced by the MCT solution (line~2). That is, every $D_i$ consists of all the depots that belong to one specific tour $T_i$ encoded by $\x$. We then perform a merging routine which merges the tours and consequently merges the connected depot sets. This routing iterates over all combinations of $D,D'\in \D, d\in D,d'\in D'$ (lines~5-8), and chooses $(d,d'),(d',d)$, such that $c_{dd'}+c_{d'd}$ is minimized. Then $\x$ and $\dD$ are updated accordingly (lines~9, 10). For a given $\dD$ and $d\in V_D$, the notation $\dD(d)$ represents the depot component $D\in \dD$ that $d$ belongs to. \vspace{5pt}

\noindent\textbf{Step 3 (lines~6-14):} The tour $T$ is partitioned into $m$ paths $\{P_1,\ldots,P_m\}$ such that the length of every path is proportional to the length of $T$ divided by $m$. Additionally, every path $P_i$ starts and ends in a depot, but not necessarily the same one. This step is reminiscent to an algorithm presented in~\cite{FredericksonETAL76} for $m$-TSP in undirected graphs. 

\begin{algorithm}
  \nonl \textbf{Input}: Allocation graph $G_A=(V_A,E_A)$, with $V_A=V_D\cupdot V_P$, number of agents $m\geq 1$\;
  \nonl \textbf{Output}: Paths $\{P_1,\ldots, P_m\}$, such that every package is visited exactly once\;
  $\x:=\{x_{uv}\}_{(u,v)\in E_A}\gets \textsc{mct}(G_A,V_P)$\;
  $\dD:=\{D_1\ldots,D_t\}\gets \textsc{ConnectedDepots}(G_A,\x)$\;
  \While{$|\dD|>1$}{
    $c_{\textup{min}}\gets\infty, d_{\textup{min}}\gets\emptyset, d'_{\textup{min}}\gets\emptyset$\;
    \For{$D,D'\in \dD, D\neq D', d\in D,d'\in D'$}{
      \If{$c_{dd'}+c_{d'd}<c_{\textup{min}}$}{
        $c_{\textup{min}}\gets c_{dd'}+c_{d'd}$, $d_{\textup{min}}\gets d, d'_{\textup{min}}\gets d'$\;
        }
      }
      $x_{d_{\textup{min}}d'_{\textup{min}}}\gets 1, x_{d'_{\textup{min}}d_{\textup{min}}}\gets 1$\;
    $\dD\gets \dD\setminus \{\dD(d_{\textup{min}}),\dD(d'_{\textup{min}})\}\cup \{\dD(d_{\textup{min}})\cup \dD(d'_{\textup{min}})\}$;\
    }
  $T:=(d_1,p_1,\ldots,p_{\ell-1},d_\ell)\gets \textsc{GetTour}(G_A,\x)$\;
  $i\gets 1$; $j\gets 1$\;
  \For{$i=0 \textup{ to } m$}{
    $P_i\gets \emptyset, L_i\gets 0$\;
    \While{$L_i\leq \textsc{length}(T)/m \textup{ \bf  and } \ell<t$}{
      $L_i\gets L_i+c_{d_jp_j}+c_{p_jd_{j+1}}$\;
      $P_i\gets P_i\cup \{(d_j,p_j),(p_j,d_{j+1})\}$\;
      $j\gets j+1$\;
      }
    }
  \Return $\{P_1,\ldots, P_m\}$\;
  \caption{\textsc{MergeSplitTours-full}}
  \label{alg:mst}
\end{algorithm} 
\subsection{Completeness and optimality}
In preparation to the proof \Cref{thm:complete},  we have the following lemma, which states that \mct produces a collection of pairwise-disjoint tours. Henceforth we assume that that $G_A$ is strongly connected and that  $G_A(V_D)$ is a directed clique.

\begin{lemma}\label{lem:tours}
Let $\x$ be the output of $\mct(G_A,V_P)$. Then there exists a collection of vertex-disjoint tours $T_1,\ldots, T_{m'}$, such that for every $(u,v)\in E_A$ such that $x_{uv}> 0$, there exists $T_i$ in which $(u,v)$ appears exactly $x_{uv}$ times. 
\end{lemma}
\begin{proof}
  By definition of \mct, for every $p\in V_p$ there exists precisely one incoming edge $(d,p)$ and one outgoing edge $(p,d')$ such that $x_{dp}=x_{pd'}=1$. Also, note that by~\Cref{eq:mct:equal}, the in-degree and out-degree of every $d\in V_D$ are equal to each other.
  Thus, an Eulerian tour can be formed, which traverses every edge $(u,v)$ exactly $x_{uv}$ times.
\end{proof}

We are ready for the main proof. 

\begin{proof}[Proof of Theorem~\ref{thm:complete}]
  First, note that after every iteration of the ``while'' loop, the updated assignment $\x$ still represents a collection of tours. Second, this loop is repeated at most $m-1$ times; $t$, which represents the initial number of connected depots (line~2), is at most $m$, since every tour induced by \mct must contain at least one depot. 

Next, let \opt be the optimal solution to $m$-MVP. That is, there exists $m$ paths $\{P^*_1,\ldots,P^*_m\}$ which represent the solution to $m$-MVP, and for every $i\in[m]$, $|P^*_i|\leq \opt$. Observe that 
\[\sum_{i=1}^{m}|P_i|\leq m\cdot \max_{i\in [m]}|P^*_i| =m\cdot \opt,\]
where $\{P_1,\ldots,P_m\}$ is the result of \alg. Next, by definition of $\alpha$, we have that $|T|\leq m\cdot \opt + m\alpha$. Lastly, by definition of $\beta$  we have that 
\[T_i\leq |T|/m +\beta \leq \opt +\alpha+\beta. \qedhere\] 
\end{proof}

\subsection{Computational complexity}
We conclude this section with an analysis of the computational complexity of \alg. Recall that we identified the main bottleneck of \alg to be the solution computation for \mct. We proceed to prove \Cref{lem:poly}, which states such a solution to \mct can be obtained via a linear relaxation. 

\begin{proof}
  The main observation to make is that $\mct$ can be transformed into a minimum-cost circulation (MCC) problem. If all edge capacities are integral, the linear relaxation of MCC enjoys a totally unimodular constraint matrix form~\cite{AhujaETAL93}. Hence, the linear relaxation will necessarily have an integer optimal solution, which will be a fortiori an optimal solution to the original MCF problem. 

In the current representation of \mct,~\cref{eq:mct:out} does not correspond to a circulation problem. However, we can introduce a small modification to $G_A$ which would allow us to recast it as circulation. Consider the graph $G'=(V',E')$ such that 
  \begin{align*}
  V':=&V_D\cup V_P\cup \{p'|p\in V_P\},\\
  E':=&\{(d,p)|d\in V_D, p\in V_P, (d,p)\in E\} \\ &\cup \{(p',d)|d\in V_D, p\in V_P,(p,d)\in E\} \\ & \cup \{(p,p')|p\in V_P\}
  \end{align*}
  and we require that for every such $(p,p')$, $x_{pp'}=1$.
\end{proof}

\section{MAPF-TN: Additional Details}

We now elaborate on two aspects of multi-agent pathfinding with transit networks (MAPF-TN)
that we alluded to in~\Cref{sec:mapf-tn}. First, we 
discuss how we extend the notion of conflict handling in Conflict-Based Search to the
capacity conflicts of MAPF-TN, where more than one agent can use a transit edge.
Second, we discuss two important speedup techniques that improve the empirical
performance of our MAPF-TN solver, without sacrificing bounded suboptimality.

\subsection{Capacity Conflicts in (E)CBS}
\label{sec:appendix-capacity}

In the classical MAPF formulation, at most one agent can occupy a particular vertex
or traverse a particular edge at a given time. Therefore, conflicts between
$p > 1$ agents yield $p$ new nodes in the multi-agent level search tree of 
Conflict-Based Search (CBS) and any of its modified variants.
In MAPF-TN, however, transit edges in general
have capacity $\capacity(e) \geq 1$.
Consider a solution generated during a run of Enhanced CBS that has assigned
to some transit edge $p > \capacity(e) > 1$ drones. 
In order to guarantee bounded sub-optimality of the solution, we must generate
all $\binom{p}{p - c}$ sets of constraints, where $c = \capacity(e)$.
Each such set of $(p - c)$ constraints represents one subset of
$(p - c)$ agents being restricted from using the transit edge in question.

As we pointed out in~\Cref{sec:results-ecbs}, conflict resolution is a significant bottleneck
for solving large MAPF-TN instances. 
In our experiments, we generated all constraint subsets of a capacity conflict, however, pathological scenarios may arise where this significantly degrades performance in practice
(our ultimate yardstick). Whether there exists a principled way to analyse constraint set enumeration and suboptimality and how this can be efficiently implemented in practice
are both important questions for future research.

\subsection{Speedup Techniques}
\label{sec:appendix-speedup}

The NP-hardness of multi-agent path finding~\cite{yu2013structure} and the additional computational challenges of MAPF-TN 
(path energy constraint; large and dense graphs)
make empirical performance paramount, given our real-world scenarios and emphasis on scalability. 
We now discuss some speedup techniques that improve the efficiency of the low-level search
while maintaining its bounded sub-optimality (which in turn ensures bounded sub-optimality of the overall solution,
as per Enhanced CBS). Certainly, these techniques are not exhaustive; there 
is an entire body of work in transportation planning devoted to speeding up 
algorithm running times~\cite{delling2009engineering}. We devised and implemented two
simple methods.

\subsubsection{Preprocessing Public Transit Networks}
Focal-MCSP can become a bottleneck when it has to be run multiple times 
(at least twice for each agent's current $\depot \package \depot^{\prime}$ task
and more in case of conflicts).
Its performance depends significantly on the availability and quality of admissible heuristics,
i.e., heuristics that \emph{underestimate} the cost to the goal, for the 
objective (elapsed time) and constraint (distance traversed, a surrogate for
the energy expended).
The public-transit network for a given area is usually known in advance and follows a pre-determined
timetable. We can analyze and preprocess 
such a network to obtain admissible heuristics.
These can then be used for multiple instances of MAPF-TN throughout the day,
while searching for paths to a specific package delivery location.

For the objective function, i.e., the elapsed time, a lower bound is typically the time to fly
directly to the goal, without deviating and waiting to board public transit
(of course, taking such a route in practice is usually infeasible due to the
distance constraint). Therefore, we define the heuristic simply as:
\begin{equation}
\label{eq:time-heur}
h_{\timeCost}(v,vg) = \frac{\lVert vg - v \rVert}{\speed}
\end{equation}
where $\speed$ is the average drone speed, $vg$ is the goal node and $v$ is the node being expanded.
The above heuristic will be admissible, i.e., be a lower
bound on elapsed time if the average drone speed is higher than average
transit speed. This assumption is typically
true for the transit vehicles we consider, given that they are
required to wait at stops for people to get on.
A more data-driven estimate can be obtained by analyzing actual flight times,
but that is out of the scope of this work.

For the constraint function, i.e., the distance traversed, we use a
heuristic based on extensive network preprocessing.
For a given transit network in the area of operation, we consider the minimal time window such that
every instance of a transit vehicle 
trip in that network can start and finish (as per the timetable).
We then create the so-called \emph{trip metagraph}, whose set
of vertices is $\depotVerts \cupdot \packageVerts \cupdot V_{\T}$,
where, from our earlier notation, $\depotVerts$ and $\packageVerts$ are the sets of depot
and package vertices respectively. Each vertex $v_{\tau} \in V_{\T}$ represents a single transit
vehicle trip $R_{\tau}$, and encodes its sequence of time-stamped stops (we will discuss what this
means in practice shortly). The trip metagraph is complete, i.e., there
is an edge between every pair of vertices.

We now define the cost of energy expenditure, i.e., distance traversed for each edge $e = (u,v)$, hereafter denoted as 
$e = (u \rightarrow v)$, in the trip metagraph. 
If $u,v \in V_{\T}$ correspond to trips $R_{\tau}$ and $R_{\tau'}$ respectively,
\begin{equation*}
\begin{aligned}
\energy(e) &= \min_{u \in R_{\tau}, v \in R_{\tau'}} \lVert v - u \rVert, \text{such that }\\
& \speed \times (v.t - u.t) \geq \lVert v - u \rVert,
\end{aligned}
\end{equation*}
where, as before, $v.t$ refers to the time-stamp of the stop $v$ for that particular trip.
The edge cost here is thus \emph{the shortest distance between stops that can be traversed by the drone in the difference
between time stamps}.
If $u,v \in \depotVerts \cupdot \packageVerts$, we simply set $\energy(e) = \lVert v - u \rVert$,
the direct flight distance between the locations. For all other edges, i.e., where one of $u$ or $v$
corresponds to a trip $R_{\tau}$ and the other to a depot or package location (in either direction),
we set
\[
\energy(e) = \min_{u \in R_{\tau}} \lVert v - u \rVert
\]
and in such cases, the cost for edge $(v \rightarrow u)$ is equal to that of $(u \rightarrow v)$. 
This concludes the assignment of edge costs.

Given the complete specification of the edge cost function,
we now run 
Floyd-Warshall's algorithm~\cite{CormenETAL09} on the trip metagraph to get 
a cost matrix $\bar{\energy}_{\T}$, where
$\bar{\energy}_{\T}(u, v)$ is the cost of the shortest-path from $v_{\tau}$ to $v_{\tau'}$ 
on the trip metagraph. Intuitively, this cost matrix encodes the \emph{least flight distance} required
to switch from one trip to another, from a trip to a depot/package and vice versa, and between two
depots/package locations, either using the transit network or flying directly, whichever is shorter.

We can now define the goal-directed heuristic function $h_{\energy}$
for the distance traversed. 
Let the goal node for a query to Focal-MCSP be $vg \in \depotVerts \cupdot \packageVerts$.
We want the heuristic value for the operation graph node $v \in V_O \equiv \left(\depotVerts \cupdot \packageVerts \cup \transitVerts \right)$ that is
expanded during Focal-MCSP. If $v \in \depotVerts \cupdot \packageVerts$ is a depot or package, we set
$h_{\energy}(v,vg) = \bar{\energy}_{\T}(v,vg)$.
Otherwise, $v \in \transitVerts$ is a transit vertex.  Recall that each transit vertex is a stop that is associated with a corresponding
transit trip. Let the trip associated with $v \in \transitVerts$ be $R_{\tau}$. We then set
$h_{\energy}(v,vg) = \bar{\energy}_{\T}(v_{\tau},vg)$, where $v_{\tau} \in V_{\T}$ is the trip metagraph
vertex corresponding to the trip $R_{\tau}$.
The heuristic $h_{\energy}$ as defined above is admissible, i.e. is a lower bound on the drone's flight distance
from the expanded operation graph node
to the target depot/package location.

In practice, we will solve several instances of MAPF-TN throughout a day, with traffic delays and other disruptions to the timetable.
However, the handling of dynamic networks and timetable delays is a separate subfield of research in transportation planning~\cite{delling2009engineering,bast2016route}
and out of the scope of this work. We make the reasonable assumption (made often in transit planning work) that travel times between locations
do not vary greatly throughout the day, and we ignore the effect of delays and disruptions to the pre-determined timetable
while using our heuristics.

\begin{figure*}[t]
  \centering
  \begin{subfigure}{0.49\textwidth}
    \centering
    \includegraphics[width=0.95\columnwidth]{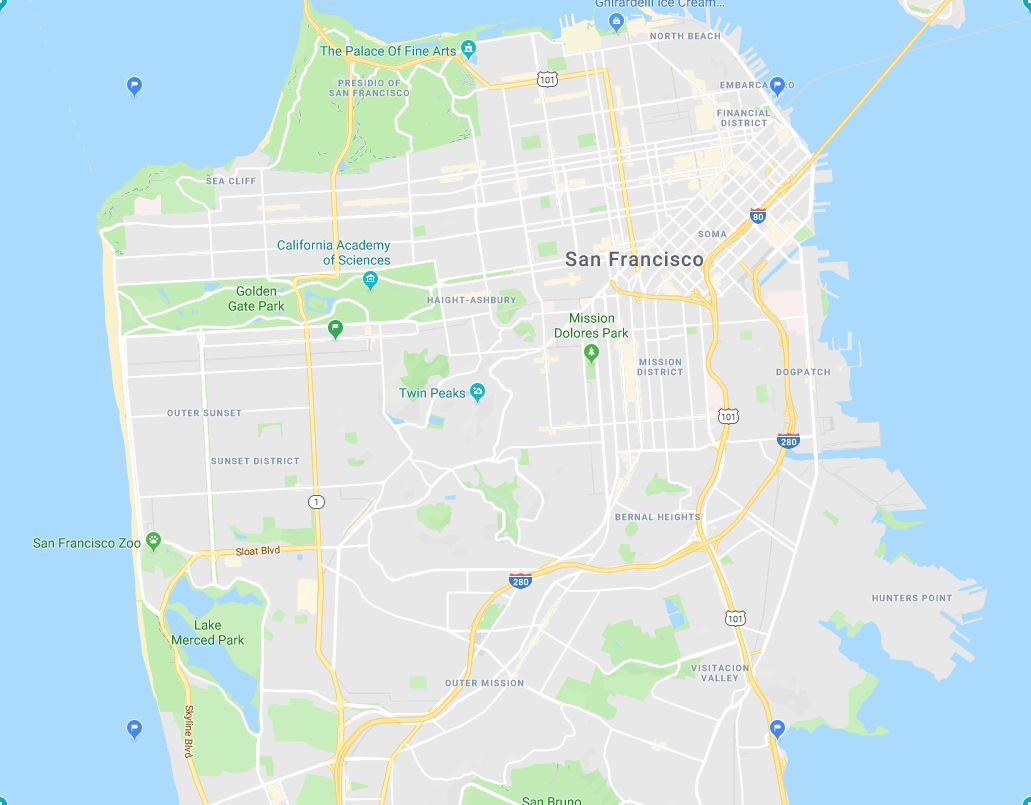}
    \caption{}
    \label{fig:sf}
  \end{subfigure}
  \begin{subfigure}{0.49\textwidth}
    \centering
    \includegraphics[width=0.9\columnwidth]{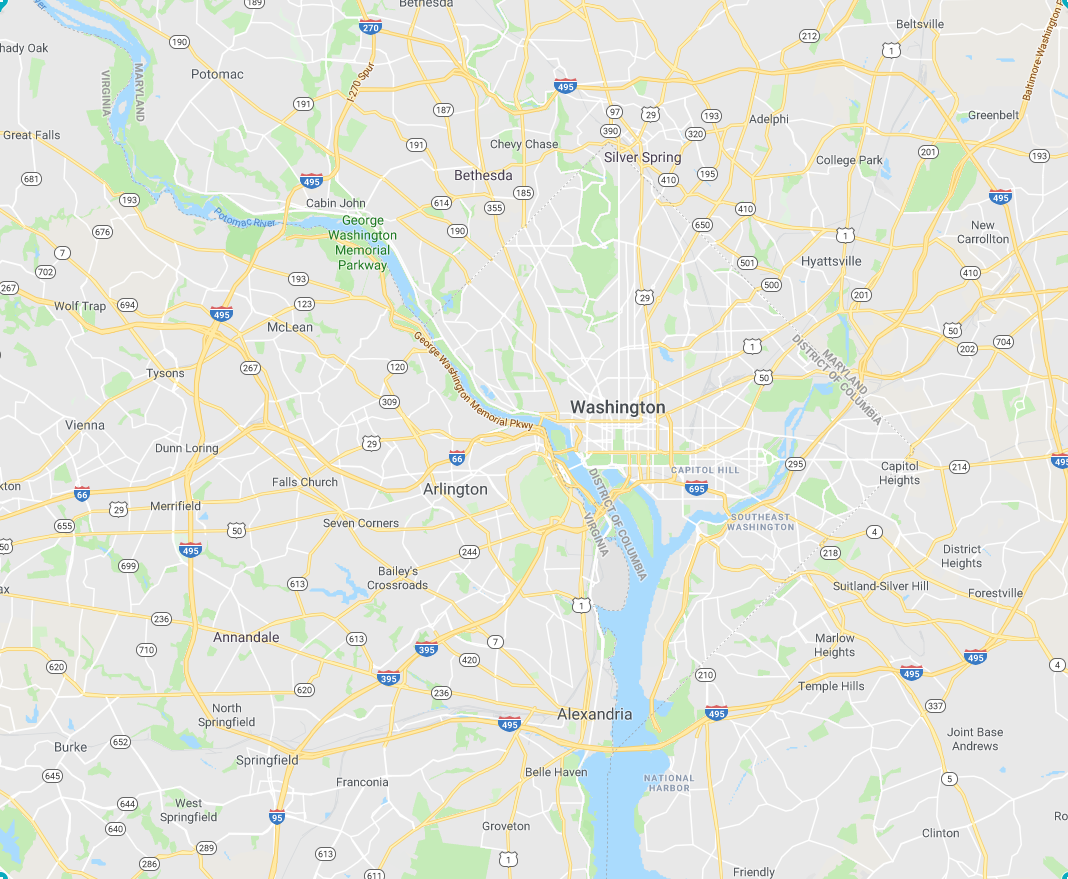}
    \caption{}
    \label{fig:wdc}
  \end{subfigure}
  \caption{The geographical bounding boxes for (\subref{fig:sf}) San Francisco (roughly \SI{150}{\kilo \metre \squared}) and the (\subref{fig:wdc}) Washington DC Metropolitan area (roughly \SI{400}{\kilo \metre \squared}).}
  \label{fig:maps}
\end{figure*}

\subsubsection{Pruning Focal-MCSP search space}

As we mentioned in~\Cref{sec:mapf-tn-formulation}, the edges of the operation graph
are not explicitly enumerated but rather implicitly encoded and generated just-in-time
during the node expansion stage of Focal-MCSP. An implicit edge set makes the Focal-MCSP search 
highly memory-efficient
by only having to store the vertices of the operation graph. This memory-efficiency comes at the
cost of computation time as the outgoing edges of a vertex must be computed during the search.
A careful observation of the transit vertices allows us to prune the set of out-neighbors
of a vertex expanded during Focal-MCSP, \emph{while still guaranteeing bounded sub-optimality}.

Let $u \in V_{O}$ be an operation graph vertex that is expanded during Focal-MCSP. Consider all the
transit vertices of a transit trip $R_{\tau}$ (if $u \in \transitVerts$ is itself a transit vertex,
then consider a trip different from the trip that $u$ lies on).
Those transit vertices are candidate out-neighbors for the expanded node $u$ (candidate target vertices of a 
\emph{time-constrained flight edge} emanating from $u$ and making a connection
to the trip $R_{\tau}$).
It may appear that all trip stops in $R_{\tau}$ that
the drone can reach in time, i.e., for which $\speed \times (v.t - u.t) \geq \lVert v - u \rVert$ (a required condition, as we mentioned 
in~\Cref{sec:mapf-tn-formulation}) should be added as out-neighbors.

However, while considering connections to a trip $R_{\tau}$, we actually need to only add the transit vertices 
on $R_{\tau}$ that are \emph{non-dominated}
in terms of the tuple of time difference and flight distance $\left(v.t - u.t, \lVert v - u \rVert\right)$.
Doing so will continue to ensure bounded sub-optimality of Focal-MCSP. We formalize this observation
through a lemma.

\begin{lemma}
Let $u \in V_O$ be an operation graph vertex expanded during Focal-MCSP. While considering time-constrained
flight connections to trip $R_{\tau}$, let $v_1$ and $v_2$ be two consecutive transit vertices on trip $R_{\tau}$ such that
$\left(v_1.t - u.t, \lVert v_1 - u \rVert\right) \preceq \left(v_2.t - u.t, \lVert v_2 - u \rVert\right)$.
Then, pruning $v_2$ as an out-neighbor has no effect on the solution of Focal-MCSP.
\end{lemma}

The following proof relies heavily on the analysis of A*-MCSP (see, e.g.,~\cite[Section V]{li2007fast}), upon which Focal-MCSP is based.
\begin{proof}
We assume that both $v_1$ and $v_2$ are physically reachable by the drone, i.e., $\speed \times (v.t - u.t) \geq \lVert v - u \rVert$
for $v = v_1, v_2$ (otherwise they would be discarded anyway).

Note that $(v_1 \rightarrow v_2)$ is a transit edge, so the flight distance $\energy(v_1 \rightarrow v_2) = 0$ by definition.
The Focal-MCSP algorithm tracks the \emph{objective} (traversal time) and \emph{constraint} (flight distance)
values of partial paths to nodes. It discards a partial path that is dominated by another on both metrics.
Consider the only two possible partial paths to $v_2$ from $u$, $u \rightarrow v_2$ and $u \rightarrow v_1 \rightarrow v_2$.

Let the weight constraint accumulated on the path thus far to the expanded node $u$ be $Wu$. The traversal time
cost at $v_2$ for both partial paths is $v_2.t$ (since $v_2$ is time-stamped). 
The accumulated traversal distance weight
at $v_2$ for $u \rightarrow v_2$ is $Wu + \energy(u \rightarrow v_2) = Wu + \lVert v_2 - u \rVert$. 
On the other hand, for $u \rightarrow v_1 \rightarrow v_2$,
the corresponding accumulated weight at node $v_2$ is $Wu + \energy(u \rightarrow v_1) + \energy(v_1 \rightarrow v_2) = Wu + \lVert v_1 - u \rVert < Wu + \lVert v_2 - u \rVert$,
by the original assumption. Focal-MCSP will always discard the partial path $u \rightarrow v_2$
and instead prefer the alternative, $u \rightarrow v_1 \rightarrow v_2$. Therefore, pruning $v_2$ as an out-neighbor
will have no effect on the solution of Focal-MCSP, which thus continues to be bounded sub-optimal.
\end{proof}
The intuition is that if a transit connection is useful to make, then a stop that is both earlier and closer in distance than
another will always be preferred.
The above result was for \emph{consecutive} vertices on a transit trip; we can extend it to the full sequence of vertices
on the trip $R_{\tau}$ by induction.

We use Kung's algorithm~\cite{kung1975finding} to find the non-dominated elements of the set of transit trip vertices.
For two criteria functions, as in our case, Kung's algorithm yields a solution in $\O(n \log n)$ time, where $n$ is the size 
of the set and the bottleneck is due to sorting the set as per one of the criteria.
In our specific case, since the transit trip vertices are already sorted in increasing order of the traversal time criterion,
we can add out-neighbors for a transit trip in $\O(n)$ time, which is as fast as we could have done anyway.

\section{Surrogate Travel Time Estimate}
\label{sec:appendix-surrogate}

At the end of~\Cref{sec:methodology}, we mentioned the role of the surrogate estimate for
travel time between two depots/packages used by \alg for the task allocation. We also briefly discussed
the actual surrogate estimate we use in our approach. We now provide some
more details about how the estimate is actually computed in a preprocessing
step and then used during runtime. In \supp{\ref{sec:appendix-results}}, we 
quantitatively compare the surrogate estimate that we use to the direct flight time
between two locations, in terms of the computation time and solution quality 
of the MAPF-TN layer.

Consider the given geographical area of operation, encoded as a bounding box of coordinates (\Cref{fig:maps}
illustrates both areas).
During preprocessing, we generate a representative set of locations across the area. To ensure good
coverage, we use a quasi-random low dispersion sampling scheme~\cite{halton1960efficiency}
to compute the locations. This set of locations induces a Voronoi decomposition~\cite{CormenETAL09} of the
geographical area where the locations are the sites. Every point
in the bounding box is associated with the nearest
element (by the appropriate distance metric) in the set of locations. We then
choose a representative time window of transit for the area. Between every pair of locations
in the set, we compute and store the travel time using the transit network (with the same Focal-MCSP parameters
we use for MAPF-TN).

During runtime, at the task allocation layer, we need the estimated travel time between two depot/package locations
$v,v' \in \depotVerts \cupdot \packageVerts$. Each of $v$ and $v'$ has a corresponding nearest representative
location (the site of its Voronoi cell). We then look up the precomputed travel time estimate between
the corresponding sites and use that value in \alg. The implicit assumption is that 
the travel time between the representative sites is the dominating factor compared to the last-mile travel between
each site and its corresponding depot/package. If $v$ and $v'$ are in the same cell, i.e., their 
nearest representative location is the same, we use the direct flight time between $v$ and $v'$,
i.e., $\lVert v' - v \rVert / \speed$. The assumption here is that $v$ and $v'$ are more likely to share a cell
if they are close together, and in that case, the drone is more likely to be able to fly directly
between them anyway.

The number of representative locations for a given area is an engineering parameter. 
For our results, we use $100$ points in San Francisco and $150$ points in Washington DC.
For the quasi-random
sampling scheme we use, the higher the number of sampled points, the lower the dispersion, i.e., the better
the coverage of the area, and typically, the better is the quality of the surrogate estimate. 
Domain knowledge about the transit network and travel time distribution in a given urban area may yield a higher quality
surrogate than our domain-agnostic approach.

\section{Further Results}
\label{sec:appendix-results}

We now elaborate on three additional aspects of our results, as we alluded to in~\Cref{sec:results}. First, 
we provide a more extensive analysis of the behavior of our layer for
multi-agent path finding with transit networks (MAPF-TN). Second, 
we compare two different replanning strategies to solve for a sequence of drone delivery tasks.
Third, we quantitatively compare the effect of two different surrogate travel time estimates.

\subsection{Further Insights of MAPF-TN Results}
\label{sec:appendix-results-mapf}

We will now supplement our discussion in~\Cref{sec:results-ecbs} on prominent observations of the behavior of
the MAPF-TN layer, based on the numbers in~\Cref{table-mapf-res}.
With regards to scalability, recall that each low-level search is actually \emph{two} Focal-MCSP searches
(from $\depot \to \package$ and $\package \to \depot'$) that are concatenated, so the effective number
of agents (from a typical MAPF perspective) is actually $2m$ and not $m$.
This observation only serves to strengthen our scalability claim.
Since our MAPF-TN solver is built upon Conflict-Based Search, the key factor affecting 
plan time is the generation and resolution of conflicts, which we have discussed
in detail already. 
We also discussed how the number of depots and the ratio of depots to agents
affects the likelihood of conflicts.
Depots or warehouses are highly expensive to construct in practice.
Thus, in a given area, the placement of
depots (that we generate randomly for our benchmarks) can have a significant impact
on computation time and scalability; indeed, that is a key question for future work.

\begin{table}[t]
\caption{(All times are in seconds) A comparison of replanning strategies for a subset of the $\{l,m\}$ scenarios from~\Cref{table-mapf-res}
for the San Francisco network. We run $20$ different trials for each setting and depict the average values in each case.}
\centering
\begin{tabular}{@{} lrrcrr @{}}
  \toprule
  & \multicolumn{2}{c}{Replan-$1$} & \phantom{abc} & \multicolumn{2}{c}{Replan-$m$}\\
  \cmidrule{2-3} \cmidrule{5-6} $\{l,m\}$ & Replan     & Soln.       & & Replan   & Soln.\\
                                            & Time      & Mksp.      & & Time    & Mksp. \\
  \midrule
  $\{5, 10\}$   & $0.271$   & $2943.1$   & & $0.645$   & $2880.1$ \\
  $\{5, 20\}$   & $0.034$   & $3092.2$   & & $1.599$   & $3092.2$ \\
  $\{20, 50\}$  & $0.006$   & $1463.5$   & & $0.278$   & $1463.5$ \\
  $\{20, 100\}$ & $0.009$   & $1952.2$   & & $0.399$   & $1952.2$ \\
  \bottomrule
\end{tabular}
\label{table-replan-res}
\end{table}

The order of magnitude higher runtimes for Washington DC is worth commenting on a bit more.
Note that we are using the same drone parameters and transit capacity settings for Washington DC, which has an area
nearly three times that of SF, and a transit network nearly twice as big.
Consequently, the need for using transit to satisfy deliveries
is greatly increased (notice how the average transit usage is reliably higher than for SF).
Additionally, the bus network for Washington DC is more sparse in the outskirts and suburban areas.
Thus, the bus network becomes more of a bottleneck than for San Francisco, leading to more conflicts.
Even when there are no conflicts, the average Focal-MCSP search times increase because more of the larger transit
graph is being explored by the search algorithm.

With regards to solution quality (makespan), we briefly commented on the real-world significance that even for a large
metropolitan area of \SI{400}{\kilo \metre \squared}, the longest delivery in a set of $m$ tasks is under $2$ hours.
We used a representative transit window that is largely replicated throughout the rest of the day;
therefore, for a given business day of, say, $12$ hours, we can expect any drone to make \emph{at least} $6$ deliveries (and typically
many more).

\begin{table*}[th]
    \caption{We compare our MAPF-TN results from~\Cref{table-mapf-res} (Average Plan Time and Makespan)
    against those where the framework uses the direct flight time as a surrogate estimate for \alg instead of our
    preprocessed surrogate using representative locations. For clarity of viewing, we split out the results
    by city/network into two separate tables. \textbf{The values for the Preprocessed sub-table are 
    copied over from~\Cref{table-mapf-res}}.}
    \begin{minipage}{.5\textwidth}
      \centering
        \caption*{\textbf{San Francisco}}
        \begin{tabular}{@{} lrrcrr @{}}
          \toprule
          & \multicolumn{2}{c}{Preprocessed} & \phantom{abc} & \multicolumn{2}{c}{Direct Flight}\\
          \cmidrule{2-3} \cmidrule{5-6} $\{l,m\}$ & Plan     & Soln.       & & Plan   & Soln.\\
                                                    & Time      & Mksp.      & & Time   & Mksp. \\
          \midrule
          $\{5, 10\}$     & $1.17$    & $2554.7$    & & $1.51$    & $2624.8$ \\
          $\{5, 20\}$     & $2.13$    & $2886.8$    & & $2.69$    & $3092.9$ \\
          $\{5, 50\}$     & $3.89$    & $3380.9$    & & $5.08$    & $3412.4$ \\
          $\{10, 20\}$    & $1.02$    & $2091.6$    & & $0.83$    & $1868.9$ \\
          $\{10, 50\}$    & $1.46$    & $2504.7$    & & $1.25$    & $2247.3$ \\
          $\{10, 100\}$   & $7.29$    & $2971.8$    & & $3.78$    & $2649.6$ \\
          $\{20, 50\}$    & $0.46$    & $1273.6$    & & $0.27$    & $1079.1$ \\
          $\{20, 100\}$   & $1.05$    & $1642.4$    & & $0.64$    & $1371.1$ \\
          $\{20, 200\}$   & $2.10$    & $1898.5$    & & $1.43$    & $1426.2$ \\
          \bottomrule
        \end{tabular}
    \end{minipage}%
    \begin{minipage}{.5\textwidth}
      \centering
      \caption*{\textbf{Washington DC}}
        \begin{tabular}{@{} lrrcrr @{}}
          \toprule
          & \multicolumn{2}{c}{Preprocessed} & \phantom{abc} & \multicolumn{2}{c}{Direct Flight}\\
          \cmidrule{2-3} \cmidrule{5-6} $\{l,m\}$ & Plan     & Soln.       & & Plan   & Soln.\\
                                                    & Time     & Mksp.      & & Time   & Mksp. \\
          \midrule
          $\{5, 10\}$     & $5.65$    & $5167.3$    & & $13.6$    & $4654.7$ \\
          $\{5, 20\}$     & $13.1$    & $5384.5$    & & $35.2$    & $5339.6$ \\
          $\{5, 50\}$     & $28.9$    & $6140.2$    & & $51.1$    & $6323.4$ \\
          $\{10, 20\}$    & $4.67$    & $4017.2$    & & $11.9$    & $4527.3$ \\
          $\{10, 50\}$    & $15.8$    & $5312.3$    & & $28.6$    & $5509.6$ \\
          $\{10, 100\}$   & $26.2$    & $5623.9$    & & $53.8$    & $5774.1$ \\
          $\{20, 50\}$    & $1.92$    & $3571.8$    & & $8.49$    & $4058.1$ \\
          $\{20, 100\}$   & $5.24$    & $4304.5$    & & $22.8$    & $4613.9$ \\
          $\{20, 200\}$   & $10.5$    & $5085.6$    & & $17.6$    & $5216.1$ \\
          \bottomrule
        \end{tabular}
    \end{minipage}
\label{table-surrogate-comp}
\end{table*}

\subsection{Replanning Strategies}
\label{sec:appendix-results-replanning}

We have previously discussed how our MAPF-TN solver based on Enhanced Conflict-Based Search (ECBS)
computes paths for a single $\depot \package \depot^{\prime}$ task for each drone. 
However, drones will typically be assigned to a sequence of deliveries by the task allocation
layer. Rather than computing paths for the entire sequence for each drone ahead of time, we
use a receding horizon approach where we replan for a drone after it completes its current task.
Our computation time is negligible compared to the actual solution execution time (compare
the `Plan Time' and `Makespan' columns in~\Cref{table-alloc}); therefore, a receding horizon
strategy appears to be quite reasonable.

Two natural replanning strategies emerge in such a context: replanning \emph{only} for the finished drone,
while maintaining the paths of all the other drones, which we call $\text{Replan-}1$,
and replanning for all drones, from each of their current states, which we call $\text{Replan-}m$.
In terms of the tradeoff between computation time and solution quality, these two approaches
are at the opposite ends of a spectrum. The $\text{Replan-}m$ strategy will be optimal among replanning
strategies, while being the most computationally expensive as it recomputes $m$ paths;
on the other hand, $\text{Replan}-1$ requires only the computation of a single path with the
remaining $m-1$ paths imposing boarding and capacity constraints.

To evaluate the two replanning strategies, we use the same setup that we did for evaluating
MAPF-TN in~\Cref{sec:results}. For each MAPF-TN solution (one path for each drone), we consider
the drone that finishes first among the $m$ drones (since we use a continuous time representation,
ties are highly unlikely in practice). In the case of $\text{Replan-}1$, we run Focal-MCSP
for the drone with the various constraints induced by the remaining paths of 
the other agents. We update the ($m$-agent) solution with the new path (updating makespan if need be).
In the case of $\text{Replan-}m$, we run Enhanced CBS for the $m$ agents with their current states 
(at that time) as their initial state; this yields another ($m$-agent) solution.

In~\Cref{table-replan-res}, we compare the average makespan and computation times of the $m$-agent
solutions resulting from the two strategies. We use a representative subset of the $\{\text{Depots, Agents}\}$
scenarios that we used in~\Cref{table-mapf-res}; few depots with a lower agent/depot ratio $(\{5,10\})$;
few depots with a higher ratio $(\{5,20\})$; and similarly for many depots $(\{20,50\} \text { and } \{20, 100\})$.
It is clear that $\text{Replan-}1$ achieves similar quality solutions as $\text{Replan-}m$
does, at fairly lower computational cost.  
This motivates our decision to use
$\text{Replan-}1$ in practice.

In principle, we can design scenarios where $\text{Replan-}1$ has a much greater solution quality gap
against $\text{Replan-}m$ than what we see in~\Cref{table-replan-res}.
However, the $\text{Replan-}1$ strategy is sub-optimal only when (i) the $(m-1)$ unfinished
drone paths actually conflict with the new Focal-MCSP path of drone $i$, that has just finished, and (ii) resolving the conflict(s)
would have prioritized the path of drone $i$ over the others.
In practice, it is not very likely that both of these conditions will hold together, especially
when there are many depots and some drones can fly directly to their next target; in our
trials for $ l = 20$, the sub-optimality condition for $\text{Replan-}1$ never holds, which
is why the makespans for those two rows are exactly the same for both strategies.

\subsection{Comparison of Surrogate Estimates}
\label{sec:appendix-results-surrogate}

We now compare the effect of two different surrogate travel time estimates --- the approximate travel time
between representative locations in the city using the transit (as described in~\supp{\ref{sec:appendix-surrogate}})
and the direct flight time between two locations, ignoring the transit.
For the results in~\Cref{table-mapf-res}, recall that we ran MAPF-TN on the first $\depot \package \depot'$ task
for each drone obtained from the result of \alg; for those results, \alg used
the preprocessed surrogate for the allocation graph edge costs.
As a comparison, we rerun the exact same scenarios as in~\Cref{table-mapf-res}, but this time, we use
the direct flight time (ignoring the transit) as the edge cost for \alg.
We compare the two primary performance factors, plan time and solution makespan, for both surrogates in~\Cref{table-surrogate-comp}.

We expect the direct flight time surrogate to be a poor estimate in scenarios where transit is used frequently,
because the allocation step does not account for it. 
Accordingly, we do observe a difference in plan time and solution quality between Preprocessed and Direct Flight for the settings with fewer depots and higher agent-to-depot ratios.
For the settings with $5$ depots in San Francisco, and for almost all settings in Washington (except the first), both computation time and the makespan are
lower for Preprocessed, i.e., it is strictly better than Direct Flight.
However, for the settings in San Francisco with $10$ or more depots, in most cases the drones are close enough to their deliveries to fly directly (recall the lower average
transit usage of those cases from~\Cref{table-mapf-res}).
Here the Direct flight surrogate is more accurate, leading to lower makespan solutions. The key takeaway
is how the choice of surrogate plays a role on real-world settings for our two-stage approach.

\end{document}